\documentclass{article}

\PassOptionsToPackage{sort,compress}{natbib}
\usepackage[preprint]{neurips_2026}

\usepackage[utf8]{inputenc}
\usepackage[T1]{fontenc}
\usepackage{amsmath,amssymb,amsthm}
\usepackage{amsfonts}
\usepackage{mathtools}
\usepackage{booktabs}
\usepackage{algorithm}
\usepackage{algorithmic}
\usepackage{hyperref}
\usepackage{cleveref}
\usepackage{url}
\usepackage{nicefrac}
\usepackage{microtype}
\usepackage{enumitem}
\usepackage{xcolor}
\usepackage{thmtools}
\usepackage{tikz}
\usetikzlibrary{positioning,arrows.meta,calc}
\usetikzlibrary{positioning,arrows.meta}

\newtheorem{theorem}{Theorem}[section]
\newtheorem{proposition}[theorem]{Proposition}
\newtheorem{corollary}[theorem]{Corollary}
\newtheorem{lemma}[theorem]{Lemma}

\newtheorem{assumption}[theorem]{Assumption}

\title{\textbf{Who Trains Matters: Federated Learning under Enrollment and Participation Selection Biases}}

\author{
Gota Morishita\\
Centre for Brain, Mind and Markets\\
University of Melbourne\\
Parkville, Victoria, Australia\\
\texttt{gota.morishita@gmail.com}
}

\begin{document}

\maketitle


\begin{abstract}
Federated learning (FL) trains a shared model from updates contributed by distributed clients, often implicitly assuming that contributing clients are representative of the target population.
In practice, this representativeness assumption can fail at two distinct stages, inducing selection bias.
First, eligibility rules such as device constraints, software requirements, or user consent determine which clients are ever enrolled and reachable for training, inducing \emph{enrollment bias}.
Second, among enrolled clients, user and system factors such as battery state, network status, and local time determine which clients participate in each communication round, inducing \emph{participation bias}.
Although existing work has largely addressed round-level participation bias, it has paid far less attention to population-level enrollment bias, which can induce a persistent mismatch between the training objective and the target-population objective.
We formalize FL under a two-stage selection model and derive \textsc{FedIPW}, an inverse-probability-weighted aggregation scheme that recovers the target-population mean update under standard ignorability and positivity assumptions.
Because client-level covariates are often unavailable for non-enrolled clients, we also introduce a limited-information aggregate-calibration extension that uses known target-population summaries to reweight the enrolled sample, partially correcting enrollment bias.
We further provide an algorithm-agnostic optimization analysis under residual weighting error and show that incomplete selection correction can induce a non-vanishing bias floor.
Finally, experiments on synthetic federated logistic regression validate the predicted objective mismatch and show that enrollment correction reduces target-population error under two-stage selection.
\end{abstract}

\section{Introduction}
\label{sec:intro}
Federated learning (FL) trains a shared model over a large population of edge devices, such as mobile phones, without requiring raw data to leave each device \citep{kairouz2021advances,konecny2016federated,mcmahan2017communication}.
In a typical FL system, the server broadcasts the current model to a subset of clients in each round and aggregates the updates they return.
Standard algorithms such as FedAvg implicitly assume that the clients participating in each round are representative of the target population that the trained model is intended to serve \citep{mcmahan2017communication}.

In practice, however, the clients participating in FL are rarely a representative sample of the target population \citep{benarba2025bias}.
This matters because the server aggregates only the updates it observes.
When client selection is systematically related to client characteristics, local data distributions, or local gradients, naive aggregation can target a selected-client objective rather than the intended target-population objective.

This selection bias can arise at two distinct stages.
First, eligibility rules such as device constraints, software requirements, or user consent determine which clients are ever enrolled and reachable for training, inducing \emph{enrollment bias}.
Second, among enrolled clients, user and system factors such as battery state, network status, and local time determine which clients participate in each communication round, inducing \emph{participation bias}.

Although existing work has largely addressed round-level participation bias \citep{cho2022biased,sun2025debiasing}, it has paid far less attention to population-level enrollment bias.
This gap is important because enrollment bias can induce a persistent mismatch between the training objective and the target-population objective.
In particular, even if an algorithm perfectly corrects round-level participation bias, it may still optimize the wrong objective when the enrolled client pool is not representative of the target population.

To address this issue, we study how to align federated updates with a target-population objective under two-stage selection.
Our contributions are as follows.
\begin{enumerate}[leftmargin=1.5em]
    \item \textbf{Two-stage selection framework and FedIPW.}
    We formalize FL under an enrollment--participation selection model and derive \textsc{FedIPW}, an inverse-probability-weighted variant of FedAvg.
    Under the two-stage factorization together with ignorability and positivity assumptions, the weighted update is unbiased for the target-population mean update when inclusion probabilities are known, and asymptotically unbiased with consistent plug-in propensity estimates.
    
    \item \textbf{Limited-information enrollment correction.}
    Because client-level covariates are often unavailable for non-enrolled clients, we discuss a limited-information regime in which only aggregate target-population summaries are available.
    We introduce an aggregate-calibration extension that uses these summaries to reweight the enrolled sample and partially mitigate enrollment bias.

    \item \textbf{Optimization error under residual weighting errors.}
    We provide an algorithm-agnostic analysis of FL optimization under residual weighting error, covering both imperfectly estimated inclusion probabilities and structural error from omitting part of the selection mechanism.
    The resulting bound decomposes the error into a transient term, a variance term, and a non-vanishing \emph{bias floor}.
    This floor scales with the residual weighting error and with gradient heterogeneity, showing that incomplete selection correction can induce persistent target-population error.
\end{enumerate}

Empirically, we use synthetic federated logistic regression as a controlled testbed in which the target-population optimum can be computed accurately.
This allows us to measure objective mismatch directly.
The experiments show that two-stage selection can cause naive aggregation to converge to a selected-client solution, that \textsc{FedIPW} recovers the target-population objective by correcting the full enrollment--participation inclusion probability, and that aggregate calibration provides a useful partial correction when only population-level summaries are available.

\section{Related Work}

\paragraph{Selection and participation in federated learning.}
Partial participation is a defining feature of practical FL.
FedAvg and related early formulations allow decentralized, unbalanced, and non-IID data, but largely treat the participating clients in each round as exogenous \citep{mcmahan2017communication}.
Subsequent work studies biased client selection, systems heterogeneity, arbitrary participation, and correlated participation, including debiasing methods for round-level participation \citep{cho2021client,cho2022biased,ruan2021flexible,wang2022unified,sun2025debiasing}.
These works primarily address which enrolled clients participate in a round.
Our focus is complementary: we distinguish this round-level participation stage from the earlier enrollment stage that determines which clients are reachable for training at all.

\paragraph{Optimization under heterogeneous local training.}
A related but distinct line of work studies optimization bias caused by data heterogeneity, systems heterogeneity, and unequal local work, often discussed as \emph{client drift} or \emph{objective inconsistency}.
Representative approaches include proximal stabilization such as FedProx \citep{li2020fedprox}, control-variate correction such as \textsc{SCAFFOLD} \citep{karimireddy2020scaffold}, and normalized aggregation such as FedNova \citep{wang2020fednova}.
These methods reduce optimization error caused by local updates on non-IID data or heterogeneous computation.
Our focus is different: we address selection bias in \emph{which clients are represented in training}.
Because our correction operates through aggregation weights, it is complementary to drift- and variance-correction methods and can in principle be combined with them.

\paragraph{IPW, calibration, and closest work.}
Methodologically, our approach is connected to inverse-probability weighting and calibration methods for non-representative samples, including the Horvitz--Thompson estimator \citep{horvitz1952generalization}, model-assisted survey sampling \citep{sarndal1992model}, and calibration weighting using auxiliary population information \citep{deville1992calibration}.

Closest to our setting is \citet{goetze2025floss}, which treats opt-out and straggler effects as missing data and applies IPW to mitigate selection bias in FL.
Although this work recognizes both opt-out/enrollment-related selection and straggler/participation-related selection, it does not explicitly formalize them as a sequential two-stage selection process.
This distinction is important because the two stages occur in temporal order and may depend on different information: enrollment should be modeled using pre-enrollment variables, whereas round-level participation may additionally depend on post-enrollment, pre-round covariates.

Our work differs in three respects.
First, we formalize FL under a two-stage enrollment--participation model and estimate the two propensity components separately, reducing structural misspecification from collapsing temporally distinct mechanisms into a single inclusion model.
Second, we analyze the optimization effect of residual weighting error, showing that omitted or misspecified selection mechanisms can induce a persistent bias floor.
Third, we discuss an aggregate-calibration extension for the limited-information regime where client-level covariates are unavailable for non-enrolled users but target-population summaries are available.

\section{Problem Formulation and Objective Mismatch}
\label{sec:problem}
Let $\mathcal P=\{1,\dots,N\}$ denote the target population of clients.  
Each client $i\in\mathcal P$ has local loss $f_i(\theta)$, where $\theta\in\mathbb R^m$ is the global model parameter.  
The target population objective is
\begin{equation}
F(\theta):=\frac{1}{N}\sum_{i=1}^N f_i(\theta).
\end{equation}

\paragraph{Target-population mean update.}
At communication round $r\in\{1,\dots,R\}$, client $i$ returns a model update $\Delta_{i,r}$.  
For local SGD with $K$ steps and local step size $\eta_r>0$, we write
\begin{equation}
\Delta_{i,r}
:=
-\eta_r\sum_{s=0}^{K-1}\mathbf g_{i,r}^{(s)},
\end{equation}
where $s$ indexes local steps and $\mathbf g_{i,r}^{(s)}$ denotes the local stochastic gradient evaluated along the local trajectory.  
The target-population mean update at round $r$ is
\begin{equation}
\bar\Delta_r:= E\bigg[\frac{1}{N}\sum_{i=1}^N \Delta_{i,r}\bigg].
\end{equation}
If the server could aggregate updates from a representative set of clients, the resulting update would track $\bar\Delta_r$. In practice, however, client participation is systematically non-random \citep{cho2022biased, benarba2025bias}. This creates a mismatch between the update induced by observed participants and the target-population mean update. To make this mismatch explicit, we model client inclusion as a two-stage selection process.

\paragraph{Two-stage selection and inclusion factorization.}

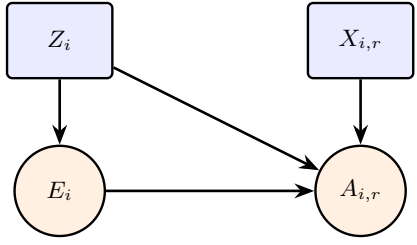
\begin{figure}[t]
\centering
\begin{tikzpicture}[
    >=Stealth,
    causal/.style={->, line width=1pt},
    obs/.style={
      draw, rounded corners=2pt,
      minimum width=14mm, minimum height=10mm,
      fill=blue!8, line width=0.9pt
    },
    sel/.style={
      draw, circle,
      minimum size=12mm,
      fill=orange!12, line width=0.9pt
    },
    font=\small
]

\node[obs] (Z) at (0,2) {$Z_i$};
\node[obs] (X) at (4,2) {$X_{i,r}$};

\node[sel] (E) at (0,0) {$E_i$};
\node[sel] (A) at (4,0) {$A_{i,r}$};

\draw[causal] (Z) -- (E);
\draw[causal] (Z) -- (A);
\draw[causal] (X) -- (A);
\draw[causal] (E) -- (A);

\end{tikzpicture}
\caption{
Causal DAG for two-stage client selection in federated learning.
Pre-enrollment attributes $Z_i$ affect enrollment $E_i$ and may also affect round-level participation $A_{i,r}$.
Pre-round covariates $X_{i,r}$ affect participation, and enrollment $E_i$ causally precedes participation.
The graph encodes the exclusion restriction that, conditional on $Z_i$, the later pre-round covariates $X_{i,r}$ do not directly affect enrollment.
}
\label{fig:two_stage_selection_dag}
\end{figure}
We distinguish between \emph{enrollment} and \emph{enrollment-conditional participation}.
Enrollment determines whether client $i$ is ever reachable for training:
\begin{equation}
E_i\in\{0,1\},
\qquad
\pi_i^{\mathrm{enroll}}:=P(E_i=1\mid Z_i),
\end{equation}
where $Z_i$ denotes pre-enrollment covariates, such as device type, OS version, and region.

Conditional on being enrolled, participation determines whether client $i$ contributes at round $r$:
\begin{equation}
A_{i,r}\in\{0,1\},
\qquad
\pi_{i,r}^{\mathrm{part}}:=P(A_{i,r}=1\mid E_i=1,Z_i,X_{i,r}),
\end{equation}
where $X_{i,r}$ denotes pre-round covariates, such as system status, battery status, local time, recent participation rate, recent connectivity status, and device-state summaries up to round $r-1$.
We take $A_{i,r}=0$ whenever $E_i=0$, so $A_{i,r}=1$ denotes overall round-$r$ inclusion in the observed aggregate.

Figure~\ref{fig:two_stage_selection_dag} summarizes the assumed two-stage causal structure.
The graph encodes the assumption that enrollment is governed by $Z_i$, while round-level participation is governed by enrollment, $Z_i$, and $X_{i,r}$:
\[
Z_i \to E_i,\quad Z_i \to A_{i,r},\quad X_{i,r}\to A_{i,r},\quad E_i\to A_{i,r}.
\]
This structure includes the exclusion restriction
\[
P(E_i=1\mid Z_i,X_{i,r})=P(E_i=1\mid Z_i),
\]
meaning that, after conditioning on pre-enrollment covariates $Z_i$, the later pre-round covariates $X_{i,r}$ do not directly affect enrollment.
Together with the fact that participation cannot occur without enrollment, the overall inclusion probability at round $r$ satisfies
\begin{align}
p_{i,r}
&:=P(A_{i,r}=1\mid Z_i,X_{i,r}) \\
&=\sum_{e\in\{0,1\}} P(E_i=e\mid Z_i,X_{i,r})
P(A_{i,r}=1\mid E_i=e,Z_i,X_{i,r}) \\
&=P(E_i=1\mid Z_i)
P(A_{i,r}=1\mid E_i=1,Z_i,X_{i,r}) \\
&=\pi_i^{\mathrm{enroll}}\pi_{i,r}^{\mathrm{part}}.
\end{align}
Thus, the overall inclusion probability factorizes into an enrollment probability and an enrollment-conditional participation probability.
This decomposition motivates modeling the two stages separately.

\paragraph{Inclusion-induced bias of naive aggregation.}
Let
\[
\mathcal S_r := \{i : A_{i,r}=1\}
\]
denote the set of clients included at round $r$. The naive server aggregate assigns equal weight to all observed clients:
\begin{equation}
\hat\Delta_r^{\mathrm{naive}}
:=
\frac{1}{|\mathcal S_r|}\sum_{i\in \mathcal S_r}\Delta_{i,r}.
\end{equation}
When inclusion depends on client characteristics through the covariate-conditional inclusion probability $p_{i,r}$, the observed sample is generally not representative of the target population. Consequently, $\hat\Delta_r^{\mathrm{naive}}$ is generally not centered on the target-population mean update $\bar\Delta_r$. In other words, selective inclusion induces an objective mismatch: the update produced by naive aggregation targets the distribution of observed participants rather than the target population $\mathcal P$.

\section{Debiasing Round-Level Updates Under Two-Stage Selection}
\label{sec:debiasing}

We now construct a weighted estimator of the target-population mean update.
Throughout this section, we assume the two-stage inclusion probability
\[
p_{i,r}=P(A_{i,r}=1\mid Z_i,X_{i,r})
=\pi_i^{\mathrm{enroll}}\pi_{i,r}^{\mathrm{part}}
\]
from Section~\ref{sec:problem}.
We also assume mean ignorability and positivity:
\[
E[\Delta_{i,r}\mid Z_i,X_{i,r},A_{i,r}=1]
=
E[\Delta_{i,r}\mid Z_i,X_{i,r}],
\qquad
p_{i,r}>0.
\]
That is, after conditioning on the observed covariates used in the selection model, inclusion does not change the conditional mean update, and every relevant client type has positive probability of being observed.

If the inclusion probabilities were known, the natural inverse-probability-weighted estimator is
\begin{equation}
\hat{\bar\Delta}_r^{\mathrm{IPW}}
:=
\frac{1}{N}\sum_{i=1}^N
\frac{A_{i,r}}{p_{i,r}}\Delta_{i,r}
=
\frac{1}{N}\sum_{i:A_{i,r}=1}
\frac{\Delta_{i,r}}{p_{i,r}} .
\label{eq:oracle-ipw}
\end{equation}

\begin{proposition}[Oracle IPW update]
\label{prop:oracle-ipw-unbiasedness}
Under two-stage factorization, mean ignorability, and positivity,
\[
E[\hat{\bar\Delta}_r^{\mathrm{IPW}}]=\bar\Delta_r .
\]
\end{proposition}
Thus, inverse-probability weighting aligns the expected aggregate update with the target-population update rather than with the selected participant distribution.
The proof is given in Appendix~\ref{app:proof-ipw-unbiasedness}.

\paragraph{Estimating the two-stage inclusion probability.}
In practice, the inclusion probability $p_{i,r}$ is unknown.
Under the two-stage factorization,
\[
p_{i,r}
=
\pi_i^{\mathrm{enroll}}\pi_{i,r}^{\mathrm{part}},
\]
we estimate the two propensity components separately.
The enrollment propensity
\[
\pi_i^{\mathrm{enroll}}=P(E_i=1\mid Z_i)
\]
is estimated from pre-enrollment covariates $Z_i$ and enrollment indicators $E_i$, when such target-population information is available.
The enrollment-conditional participation propensity
\[
\pi_{i,r}^{\mathrm{part}}
=
P(A_{i,r}=1\mid E_i=1,Z_i,X_{i,r})
\]
is estimated among enrolled clients using pre-round covariates $X_{i,r}$ and round-level participation indicators $A_{i,r}$.
For example, either component may be estimated using logistic regression or more flexible classifiers, depending on the complexity of the selection mechanism \citep{cole2008constructing}.
This gives the plug-in inclusion probability
\[
\hat p_{i,r}
=
\hat\pi_i^{\mathrm{enroll}}\hat\pi_{i,r}^{\mathrm{part}}.
\]
Estimating $\pi_i^{\mathrm{enroll}}$ and $\pi_{i,r}^{\mathrm{part}}$ separately is also important for avoiding selection-model misspecification.
Enrollment is a pre-round selection event and should be modeled using only pre-enrollment information, whereas round-level participation may depend on post-enrollment, pre-round covariates.
Collapsing the two stages into a single inclusion model can obscure this temporal order; in particular, it may use later availability variables to explain earlier enrollment decisions, fitting correlations in the observed data rather than the sequential mechanism that generated inclusion.
This form of model misspecification induces residual weighting error, which is precisely the type of error captured by $\epsilon_w$ in our optimization analysis \citep{rosenbaum1983central,robins1986new,hernan2020causal}.

\paragraph{FedIPW.}
We propose \textsc{FedIPW}, a two-stage inverse-probability-weighted FL aggregation procedure.
Before communication rounds begin, FedIPW estimates the enrollment propensity $\hat\pi_i^{\mathrm{enroll}}$ from pre-enrollment covariates and enrollment indicators, when such target-population information is available.
At each communication round $r$, FedIPW estimates the enrollment-conditional participation propensity $\hat\pi_{i,r}^{\mathrm{part}}$ among enrolled clients using pre-round covariates and round-level participation indicators.
It then forms the plug-in inclusion probability
\[
\hat p_{i,r}
=
\hat\pi_i^{\mathrm{enroll}}\hat\pi_{i,r}^{\mathrm{part}},
\]
and reweights each returned client update by the inverse of this probability.
The resulting aggregate is
\begin{equation}
\hat{\bar\Delta}_r^{\mathrm{FedIPW}}
=
\frac{1}{N}\sum_{i:A_{i,r}=1}
\frac{\Delta_{i,r}}{\hat p_{i,r}}.
\label{eq:fedipw-update}
\end{equation}

The next Proposition~\ref{prop:fedipw-plugin} shows that FedIPW remains asymptotically centered on the target-population mean update when the two propensity components are consistently estimated.

\begin{proposition}[plug-in IPW update]
\label{prop:fedipw-plugin}
Fix a communication round $r$.
Suppose the enrollment and participation propensity estimators are uniformly consistent in their estimation samples, the true and estimated inclusion probabilities are bounded away from zero, and the observed updates satisfy a bounded moment condition.
Then
\[
E[\hat{\bar\Delta}_r^{\mathrm{FedIPW}}]-\bar\Delta_r \to 0 .
\]
\end{proposition}

The formal assumptions and proof are given in Appendix~\ref{app:proof-plugin-ipw}.

\section{Limited-Information Enrollment Correction}
\label{sec:limited-info-enrollment}

FedIPW requires individual-level covariates for estimating enrollment propensities over the target population.
In many deployments, non-enrolled clients are not individually observable, while aggregate target-population summaries may be available from external sources, product analytics, or census-like statistics.
We therefore consider an aggregate-calibration extension.

The idea is to reweight the enrolled clients so that functions of their pre-enrollment covariates match known summaries of the target population.
Let $b(Z_i)$ be a vector-valued function of the pre-enrollment covariates $Z_i$, such as indicators for region, device class, language, or OS version.
Suppose the corresponding target-population average
\[
\mu_b = E[b(Z_i)]
\]
is known.
We choose weights $q_i$ on enrolled clients so that
\[
\sum_{i\in\mathcal E} q_i=1,
\qquad
\sum_{i\in\mathcal E} q_i b(Z_i)=\mu_b.
\]
These weights make the enrolled sample match the target population along the selected summaries.
For example, if enrolled clients overrepresent newer devices or a particular region, calibration increases the weight of underrepresented enrolled clients until the weighted enrolled distribution matches the known target-population proportions.

This is a limited-information correction, not a full replacement for enrollment IPW.
If the chosen covariate summaries $b(Z)$ capture all enrollment-related variation in the mean update, then calibration recovers asymptotic unbiasedness; we state this formal condition in Appendix~\ref{app:aggregate-calibration}.
In typical ML applications, this sufficiency condition is unlikely to hold exactly.
Thus, aggregate calibration should be interpreted as an approximate bias-reduction method that corrects only the component of enrollment mismatch explained by the chosen covariate summaries.

\section{Optimization Error under Weight Misspecification}
\label{sec:convergence}

The previous sections show how FedIPW aligns FL updates with the target-population update when the two-stage inclusion probabilities are known or consistently estimated.
We analyze the optimization error under the residual weighting error caused either by finite-sample propensity estimation or by ignoring or only partially correcting a selection mechanism.

Let
\[
\rho_{i,r}:=\frac{p_{i,r}}{\widehat p_{i,r}}
\]
denote the ratio between the true inclusion probability and the probability used by the algorithm.
When $\rho_{i,r}=1$, the aggregation weight is correct; deviations from one measure residual weighting error.
We assume this error is uniformly bounded:
\[
|\rho_{i,r}-1|\le \epsilon_w .
\]

This ratio has two interpretations.
First, when the algorithm uses the correct two-stage form
\[
\widehat p_{i,r}
=
\widehat\pi_i^{\mathrm{enroll}}\widehat\pi_{i,r}^{\mathrm{part}},
\]
$\epsilon_w$ captures statistical error from estimating the two propensities.
Second, when the algorithm omits part of the selection mechanism, $\epsilon_w$ captures structural misspecification.
For example, a participation-only method uses
\[
\widehat p_{i,r}=\widehat\pi_{i,r}^{\mathrm{part}}.
\]
Even if $\widehat\pi_{i,r}^{\mathrm{part}}=\pi_{i,r}^{\mathrm{part}}$, the ratio becomes
\[
\rho_{i,r}=\pi_i^{\mathrm{enroll}},
\]
so enrollment heterogeneity remains as residual weight error.

We now summarize how residual weighting error affects optimization across many rounds.
We use a strongly convex setting because it lets us quantify the resulting target-population suboptimality gap.
For nonconvex objectives, one typically obtains stationarity bounds, which do not directly measure a persistent gap to the target-population optimum.

\begin{theorem}[Bias-floor decomposition; informal]
\label{thm:fedipw-bias-floor}
Consider the FL algorithm run for $R$ communication rounds, and let
\[
F^\star:=\min_{\theta}F(\theta),
\qquad
h_0:=F(\theta_0)-F^\star .
\]
Suppose the standard smoothness, strong-convexity, stochastic-gradient, overlap, and step-size conditions stated in Appendix~\ref{app:proof-conv} hold.
Suppose also that the residual weighting error is bounded with a sufficiently small error $\epsilon_w$:
\[
|\rho_{i,r}-1|\le \epsilon_w
\qquad
\text{for all clients } i \text{ and rounds } r .
\]
Then the last iterate $\theta_R$ satisfies
\[
E[F(\theta_R)-F^\star]
\le
\widetilde O\!\left(
\frac{V}{\mu R}
+
h_0\exp(-cR)
+
\frac{\epsilon_w^2G^2}{\mu}
\right),
\]
where $V$ is a variance constant, $G$ measures client-gradient heterogeneity, $\mu$ is the strong-convexity parameter, and $c>0$ is a problem-dependent constant.
\end{theorem}

Appendix~\ref{app:proof-conv} gives the formal theorem and proof.
The theorem shows that residual weight error induces a non-vanishing bias floor whose size is controlled by $\epsilon_w^2G^2/\mu$.
The bias floor is driven by the residual weighting error $\epsilon_w$ and amplified by client-gradient heterogeneity $G$.
If $\epsilon_w=0$, the floor vanishes; if $\epsilon_w>0$, larger gradient heterogeneity $G$ makes the same weight error more damaging.

Because the theorem is agnostic to where the error $\epsilon_w$ comes from, it can be applied to different sources of misspecification.
The next corollary applies it to participation-only correction and shows that the omitted enrollment factor itself becomes the residual weighting error.

\begin{corollary}[Participation-only correction; informal]
\label{cor:participation-only-informal}
A method that corrects only round-level participation but omits enrollment has residual weight error even when the participation propensities are known exactly.
In particular, if
\[
\widehat p_{i,r}=\pi_{i,r}^{\mathrm{part}},
\qquad
p_{i,r}=\pi_i^{\mathrm{enroll}}\pi_{i,r}^{\mathrm{part}},
\]
then
\[
\rho_{i,r}
=
\frac{p_{i,r}}{\widehat p_{i,r}}
=
\pi_i^{\mathrm{enroll}},
\qquad
\epsilon_w
=
\max_i |1-\pi_i^{\mathrm{enroll}}|.
\]
Thus enrollment bias appears as structural weighting error and contributes to the bias floor in Theorem~\ref{thm:fedipw-bias-floor}.
\end{corollary}

This shows that even with perfect participation-bias correction, the omitted enrollment factor remains as the relative residual weight error inside $\epsilon_w$, so participation-bias correction alone cannot remove the bias floor when enrollment is heterogeneous.
Appendix~\ref{app:participation-only-corollary} gives the formal statement.

A natural question is whether the resulting bias floor is only an artifact of the proof, or whether the dependence on $\epsilon_w^2G^2/\mu$ is unavoidable.
The next proposition answers this question for misspecified selection-weighted objectives.

\begin{proposition}[Order-sharpness; informal]
\label{prop:weighted-objective-lower-bound-informal}
For any $\epsilon_w\in(0,1/2]$, there exists a two-client strongly convex FL instance and a misspecified selection-weighted objective whose relative weight error is exactly $\epsilon_w$ such that the minimizer $\theta_\rho^\star$ of the misspecified objective satisfies
\[
F(\theta_\rho^\star)-F^\star
\ge
c\,\frac{\epsilon_w^2G^2}{\mu}
\]
for a universal constant $c>0$.
The construction is realized by participation-only IPW when enrollment is heterogeneous and the enrollment factor is omitted.
\end{proposition}

Proposition~\ref{prop:weighted-objective-lower-bound-informal} shows that the dependence on $\epsilon_w^2$, $G^2$, and $1/\mu$ in Theorem~\ref{thm:fedipw-bias-floor} is not a proof artifact.
It is order-sharp for misspecified selection-weighted objectives, including the participation-only failure mode in Corollary~\ref{cor:participation-only-informal}.
The formal construction is given in Appendix~\ref{app:weighted-objective-lower-bound}.

\section{Experiments}
\label{sec:experiments}
\begin{figure}[t]
    \centering
    \includegraphics[width=0.8\textwidth]{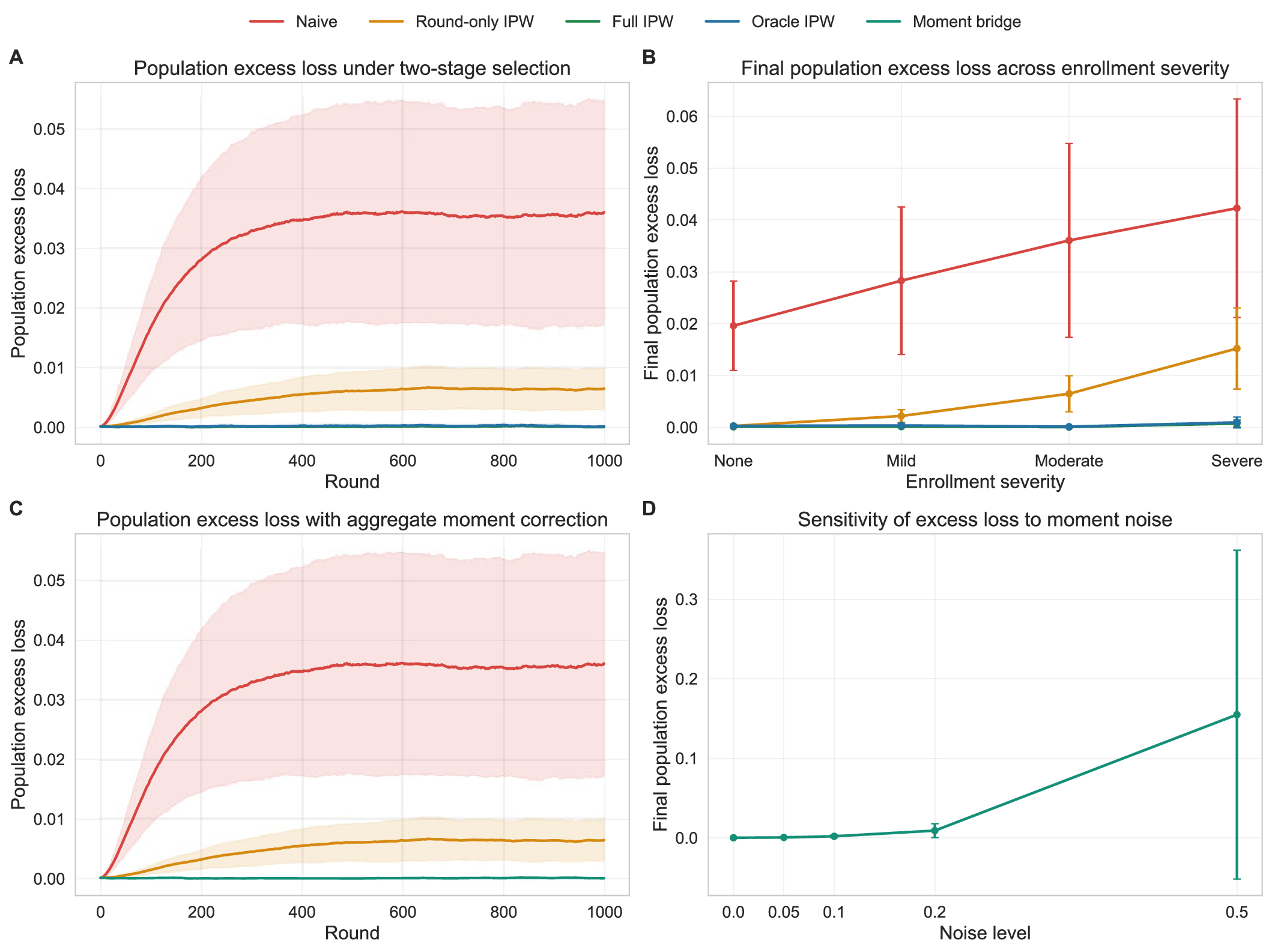}
    \caption{
    Synthetic federated logistic-regression experiments under two-stage selection.
    \textbf{A:} Target-population loss over communication rounds under two-stage selection.
    FedAvg converges to a selected-client solution, while FedIPW closely tracks Oracle FedIPW.
    \textbf{B:} Final target-population loss as enrollment bias increases.
    Participation-only IPW improves over FedAvg but degrades as enrollment bias strengthens, whereas FedIPW remains close to Oracle FedIPW.
    \textbf{C:} Target-population loss under the limited-information setting with exact aggregate population moments.
    Aggregate calibration reduces the residual mismatch left by Participation-only IPW.
    \textbf{D:} Sensitivity of aggregate calibration to noisy population moments, measured by final distance to the target-population optimizer.
    }
    \label{fig:logistic-main}
\end{figure}

We empirically study whether two-stage selection changes the objective optimized by federated learning, and whether the proposed corrections recover the target-population objective.
We also study whether the aggregate calibration method (Section~\ref{sec:limited-info-enrollment}) can still partially reduce the enrollment-induced mismatch when individual-level information about non-enrolled clients is unavailable,

We use synthetic federated logistic regression as a controlled testbed.
This setting is intentionally simple: because the target-population optimum can be computed accurately, we can directly measure whether an algorithm converges to the intended population solution rather than to a selected-client solution.
This makes it possible to isolate selection bias from other sources of error such as model misspecification or nonconvex optimization.

\subsection{Setup}

We consider a finite population of clients, each with a local logistic-regression objective.
Clients differ in pre-enrollment covariates, and these covariates affect both their local data distribution and their probability of being selected.
The selection process follows the two-stage structure in Section~\ref{sec:problem}.
An enrollment mechanism first determines which clients are reachable for federated training.
Then, among enrolled clients, a round-level participation mechanism determines which clients contribute updates in each communication round.

We compare four aggregation schemes:
\begin{itemize}[leftmargin=1.5em]
    \item \emph{Naive FedAvg} averages the observed client updates without correcting for selection.
    \item \emph{Round-only IPW} corrects participation among enrolled clients but ignores the enrollment stage.
    \item \emph{FedIPW} corrects the full two-stage inclusion probability.
    \item \emph{Oracle IPW} uses the true inclusion probabilities and serves as an idealized reference.
\end{itemize}

\paragraph{Two-stage correction recovers the target-population objective}
We first examine whether selection changes the objective optimized by federated training, and whether correcting the full two-stage inclusion probability removes this mismatch.
Figure~\ref{fig:logistic-main}A shows the target-population loss over communication rounds under two-stage selection.
FedAvg converges quickly but plateaus at a substantially higher target-population loss, indicating that the observed participants induce a selected-client objective that differs from \(F\).
Participation-only IPW reduces this mismatch by correcting round-level participation, but it still leaves a visible gap because the enrolled client pool remains non-representative.
FedIPW closes this gap and closely tracks Oracle FedIPW, showing that two-stage correction recovers the target-population objective.

Figure~\ref{fig:logistic-main}B isolates the role of enrollment by varying the strength of enrollment bias while keeping the participation mechanism fixed.
As enrollment bias increases, FedAvg and Participation-only IPW both move farther from the target-population optimum.
This shows that participation correction alone cannot remove the mismatch caused by a non-representative enrolled population.
In contrast, FedIPW remains close to Oracle FedIPW across the sweep, indicating that the two-stage correction is stable as enrollment bias becomes stronger.
The widening gap between Participation-only IPW and FedIPW directly measures the cost of omitting the enrollment stage.

\paragraph{Aggregate calibration mitigates enrollment bias under limited information}

FedIPW requires individual-level covariates for estimating enrollment propensities over the target
population. 
In many deployments, however, individual-level covariates for non-enrolled clients may be unavailable.
We therefore evaluate the aggregate-calibration approach from Section~\ref{sec:limited-info-enrollment}.
This experiment tests whether such population-level information can reduce the enrollment-induced mismatch when full two-stage IPW is infeasible.

Figure~\ref{fig:logistic-main}C shows that aggregate calibration can close much of the gap left by round-only IPW when the population summaries are accurate.
The result is not surprising in this controlled setting: the same covariates that drive enrollment also explain a substantial part of the variation in local objectives.
Matching their population moments therefore moves the effective training distribution closer to the target population.

However, calibration is only a partial correction.
Figure~\ref{fig:logistic-main}D perturbs the supplied population moments and measures the final distance to \(\theta^\star\).
As the aggregate information becomes noisier, the benefit of calibration decreases.
With mildly noisy moments, calibration still improves over round-only IPW.
With sufficiently noisy moments, the correction can become unreliable.
This confirms the intended interpretation of aggregate calibration: it is a practical fallback when full enrollment correction is infeasible, not a replacement for observing or estimating the enrollment mechanism.

\section{Discussion}
\label{sec:discussion}
The main message of this paper is that \emph{who can train at all} is as important as \emph{who participates in a round}. Existing FL debiasing work often focuses on round-level participation, but our framework shows that this can leave a structural mismatch: if the enrolled client pool is not representative of the target population, participation correction alone targets the wrong objective. By separating enrollment from participation, FedIPW turns this representativeness issue into an explicit weighting problem and clarifies which assumptions are needed for the aggregate update to represent the target population.

Our optimization analysis separates statistical propensity-estimation error from structural misspecification. Both enter through the same residual weight-error term and can create a non-vanishing bias floor, but they have different implications: estimation error can shrink with better propensity models, whereas omitting enrollment targets the wrong inclusion probability. Thus, enrollment correction is not a minor refinement of participation correction, but a separate requirement for population-aligned FL.

Finally, the limited-information calibration extension highlights a practical middle ground. Full enrollment IPW requires individual-level covariates for non-enrolled clients, which may be unavailable in real deployments. Aggregate calibration cannot remove arbitrary enrollment bias, but it provides a principled way to use population summaries that would otherwise be ignored. 

\bibliographystyle{plainnat}
\bibliography{references}

@inproceedings{mcmahan2017communication,
  author    = {H. Brendan McMahan and Eider Moore and Daniel Ramage and Seth Hampson and Blaise Ag{"u}era y Arcas},
  title     = {Communication-Efficient Learning of Deep Networks from Decentralized Data},
  booktitle = {Proceedings of the 20th International Conference on Artificial Intelligence and Statistics},
  series    = {Proceedings of Machine Learning Research},
  volume    = {54},
  pages     = {1273--1282},
  year      = {2017}
}

@article{kairouz2021advances,
  title     = {Advances and Open Problems in Federated Learning},
  author    = {Kairouz, Peter and McMahan, H. Brendan and Avent, Brendan and Bellet, Aur{\'e}lien and Bennis, Mehdi and Bhagoji, Arjun Nitin and Bonawitz, Kallista and Charles, Zachary and Cormode, Graham and Cummings, Rachel and others},
  journal   = {Foundations and Trends{\textregistered} in Machine Learning},
  volume    = {14},
  number    = {1--2},
  pages     = {1--210},
  year      = {2021},
  publisher = {Now Publishers, Inc.}
}

@article{konecny2016federated,
  title   = {Federated Optimization: Distributed Machine Learning for On-Device Intelligence},
  author  = {Kone{\v{c}}n{\'y}, Jakub and McMahan, H. Brendan and Ramage, Daniel and Richt{\'a}rik, Peter},
  journal = {arXiv preprint arXiv:1610.02527},
  year    = {2016}
}

@inproceedings{cho2021client,
  title     = {Client Selection in Federated Learning: Convergence Analysis and Power-of-Choice Selection Strategies},
  author    = {Cho, Yae Jee and Wang, Jianyu and Joshi, Gauri},
  booktitle = {International Conference on Learning Representations},
  year      = {2021},
  url       = {https://openreview.net/forum?id=PYAFKBc8GL4}
}

@inproceedings{cho2022biased,
  author    = {Yae Jee Cho and Jianyu Wang and Gauri Joshi},
  title     = {Towards Understanding Biased Client Selection in Federated Learning},
  booktitle = {Proceedings of The 25th International Conference on Artificial Intelligence and Statistics},
  series    = {Proceedings of Machine Learning Research},
  volume    = {151},
  pages     = {10351--10375},
  year      = {2022},
  publisher = {PMLR}
}

@inproceedings{sun2025debiasing,
  author    = {Zhenyu Sun and Ziyang Zhang and Zheng Xu and Gauri Joshi and Pranay Sharma and Ermin Wei},
  title     = {Debiasing Federated Learning with Correlated Client Participation},
  booktitle = {Proceedings of the 13th International Conference on Learning Representations},
  pages     = {16492--16523},
  year      = {2025}
}

@inproceedings{ruan2021flexible,
  title     = {Towards Flexible Device Participation in Federated Learning},
  author    = {Ruan, Yichen and Zhang, Xiaoxi and Liang, Shu-Che and Joe-Wong, Carlee},
  booktitle = {Proceedings of the 24th International Conference on Artificial Intelligence and Statistics},
  series    = {Proceedings of Machine Learning Research},
  volume    = {130},
  pages     = {3403--3411},
  year      = {2021},
  publisher = {PMLR}
}

@inproceedings{wang2022unified,
  title     = {A Unified Analysis of Federated Learning with Arbitrary Client Participation},
  author    = {Wang, Shiqiang and Ji, Mingyue},
  booktitle = {Advances in Neural Information Processing Systems},
  volume    = {35},
  year      = {2022},
  url       = {https://openreview.net/forum?id=qSs7C7c4G8D}
}

@inproceedings{li2020fedprox,
  title     = {Federated Optimization in Heterogeneous Networks},
  author    = {Li, Tian and Sahu, Anit Kumar and Zaheer, Manzil and Sanjabi, Maziar and Talwalkar, Ameet and Smith, Virginia},
  booktitle = {Proceedings of Machine Learning and Systems},
  volume    = {2},
  pages     = {429--450},
  year      = {2020}
}

@inproceedings{karimireddy2020scaffold,
  author    = {Sai Praneeth Karimireddy and Satyen Kale and Mehryar Mohri and Sashank Reddi and Sebastian Stich and Ananda Theertha Suresh},
  title     = {{SCAFFOLD}: Stochastic Controlled Averaging for Federated Learning},
  booktitle = {Proceedings of the 37th International Conference on Machine Learning},
  series    = {Proceedings of Machine Learning Research},
  volume    = {119},
  pages     = {5132--5143},
  year      = {2020},
  publisher = {PMLR}
}

@inproceedings{wang2020fednova,
  title     = {Tackling the Objective Inconsistency Problem in Heterogeneous Federated Optimization},
  author    = {Wang, Jianyu and Liu, Qinghua and Liang, Hao and Joshi, Gauri and Poor, H. Vincent},
  booktitle = {Advances in Neural Information Processing Systems},
  volume    = {33},
  pages     = {7611--7623},
  year      = {2020}
}

@article{horvitz1952generalization,
  author  = {Daniel G. Horvitz and Donovan J. Thompson},
  title   = {A Generalization of Sampling Without Replacement from a Finite Universe},
  journal = {Journal of the American Statistical Association},
  volume  = {47},
  number  = {260},
  pages   = {663--685},
  year    = {1952}
}

@book{sarndal1992model,
  author    = {Carl-Erik S{\"a}rndal and Bengt Swensson and Jan Wretman},
  title     = {Model Assisted Survey Sampling},
  publisher = {Springer},
  year      = {1992}
}

@article{deville1992calibration,
  author  = {Jean-Claude Deville and Carl-Erik S{\"a}rndal},
  title   = {Calibration Estimators in Survey Sampling},
  journal = {Journal of the American Statistical Association},
  volume  = {87},
  number  = {418},
  pages   = {376--382},
  year    = {1992},
  doi     = {10.1080/01621459.1992.10475217}
}

@article{rosenbaum1983central,
  title   = {The central role of the propensity score in observational studies for causal effects},
  author  = {Rosenbaum, Paul R. and Rubin, Donald B.},
  journal = {Biometrika},
  volume  = {70},
  number  = {1},
  pages   = {41--55},
  year    = {1983}
}

@article{robins1986new,
  title   = {A new approach to causal inference in mortality studies with a sustained exposure period---application to control of the healthy worker survivor effect},
  author  = {Robins, James M.},
  journal = {Mathematical Modelling},
  volume  = {7},
  number  = {9--12},
  pages   = {1393--1512},
  year    = {1986}
}

@book{hernan2020causal,
  title     = {Causal Inference: What If},
  author    = {Hern{\'a}n, Miguel A. and Robins, James M.},
  year      = {2020},
  publisher = {Chapman \& Hall/CRC}
}

@article{cole2008constructing,
  title   = {Constructing inverse probability weights for marginal structural models},
  author  = {Cole, Stephen R. and Hern{\'a}n, Miguel A.},
  journal = {American Journal of Epidemiology},
  volume  = {168},
  number  = {6},
  pages   = {656--664},
  year    = {2008}
}

@article{goetze2025floss,
  author  = {David J. Goetze and Dahlia J. Felten and Jeannie R. Albrecht and Rohit Bhattacharya},
  title   = {{FLOSS}: Federated Learning with Opt-Out and Straggler Support},
  journal = {arXiv preprint arXiv:2507.23115},
  year    = {2025}
}

@article{benarba2025bias,
  author  = {Nawel Benarba and Sara Bouchenak},
  title   = {Bias in Federated Learning: A Comprehensive Survey},
  journal = {ACM Computing Surveys},
  volume  = {57},
  number  = {11},
  pages   = {1--36},
  year    = {2025}
}
\appendix

\section{Oracle IPW Unbiasedness}
\label{app:proof-ipw-unbiasedness}

This appendix proves Proposition~\ref{prop:oracle-ipw-unbiasedness}.
Fix a round $r$ and write
\[
W_{i,r}:=(Z_i,X_{i,r}),
\qquad
p_{i,r}:=P(A_{i,r}=1\mid W_{i,r}).
\]
Assume mean ignorability and positivity:
\[
E[\Delta_{i,r}\mid W_{i,r},A_{i,r}=1]
=
E[\Delta_{i,r}\mid W_{i,r}],
\qquad
p_{i,r}>0.
\]
The oracle IPW estimator is
\[
\hat{\bar\Delta}_r^{\mathrm{IPW}}
=
\frac{1}{N}\sum_{i=1}^N
\frac{A_{i,r}}{p_{i,r}}\Delta_{i,r}.
\]
We define the target-population mean update as
\[
\bar\Delta_r
:=
\frac{1}{N}\sum_{i=1}^N E[\Delta_{i,r}],
\]
or equivalently as the expectation of the finite-population mean update.

For each client $i$,
\begin{align}
E\!\left[
\frac{A_{i,r}}{p_{i,r}}\Delta_{i,r}
\right]
&=
E\!\left[
\frac{1}{p_{i,r}}
E[A_{i,r}\Delta_{i,r}\mid W_{i,r}]
\right] \\
&=
E\!\left[
E[\Delta_{i,r}\mid W_{i,r},A_{i,r}=1]
\right] \\
&=
E\!\left[
E[\Delta_{i,r}\mid W_{i,r}]
\right] \\
&=
E[\Delta_{i,r}].
\end{align}
The second equality uses
\[
E[A_{i,r}\Delta_{i,r}\mid W_{i,r}]
=
p_{i,r}E[\Delta_{i,r}\mid W_{i,r},A_{i,r}=1],
\]
and the third equality uses ignorability.
Averaging over clients gives
\[
E[\hat{\bar\Delta}_r^{\mathrm{IPW}}]
=
\frac{1}{N}\sum_{i=1}^N E[\Delta_{i,r}]
=
\bar\Delta_r.
\]

\section{Formal Statement and Proof for Plug-in FedIPW}
\label{app:proof-plugin-ipw}

This appendix gives the formal statement corresponding to Proposition~\ref{prop:fedipw-plugin}.
Fix a communication round $r$.
Let
\[
p_{i,r}
=
P(A_{i,r}=1\mid Z_i,X_{i,r})
=
\pi_i^{\mathrm{enroll}}\pi_{i,r}^{\mathrm{part}}
\]
denote the true two-stage inclusion probability, and let
\[
\hat p_{i,r}
=
\hat\pi_i^{\mathrm{enroll}}\hat\pi_{i,r}^{\mathrm{part}}
\]
be its plug-in estimate.
The oracle IPW estimator and plug-in FedIPW estimator are
\[
\hat{\bar\Delta}_r^{\mathrm{IPW}}
:=
\frac{1}{N}\sum_{i=1}^N
\frac{A_{i,r}}{p_{i,r}}\Delta_{i,r},
\qquad
\hat{\bar\Delta}_r^{\mathrm{FedIPW}}
:=
\frac{1}{N}\sum_{i=1}^N
\frac{A_{i,r}}{\hat p_{i,r}}\Delta_{i,r}.
\]

All limits in this section are taken as the samples used to estimate the enrollment and participation propensity models grow.

\begin{assumption}[Two-stage overlap at round $r$]
\label{assump:plugin-overlap}
There exists a constant $c_r>0$ such that, for all $i\in\{1,\dots,N\}$,
\[
p_{i,r}\ge c_r,
\qquad
\hat p_{i,r}\ge c_r.
\]
\end{assumption}

\begin{assumption}[Uniform consistency of propensity components]
\label{assump:plugin-propensity-consistency}
The enrollment and participation propensity estimators satisfy
\[
E\!\left[
\max_{1\le i\le N}
|\hat\pi_i^{\mathrm{enroll}}-\pi_i^{\mathrm{enroll}}|
\right]\to 0,
\qquad
E\!\left[
\max_{1\le i\le N}
|\hat\pi_{i,r}^{\mathrm{part}}-\pi_{i,r}^{\mathrm{part}}|
\right]\to 0.
\]
\end{assumption}

\begin{assumption}[Bounded observed-update moment]
\label{assump:plugin-bounded-update}
Let
\[
U_{N,r}:=
\frac{1}{N}\sum_{i=1}^N A_{i,r}\|\Delta_{i,r}\|_2.
\]
Then
\[
E[U_{N,r}^2]=O(1).
\]
\end{assumption}

\begin{proposition}[Plug-in FedIPW consistency at a fixed round]
\label{prop:plugin-remainder}
Fix a communication round $r$.
Suppose Assumptions~\ref{assump:plugin-overlap}--\ref{assump:plugin-bounded-update} hold.
If the oracle IPW estimator is unbiased for the target-population mean update, i.e.
\[
E[\hat{\bar\Delta}_r^{\mathrm{IPW}}]=\bar\Delta_r,
\]
then
\[
E[\hat{\bar\Delta}_r^{\mathrm{FedIPW}}]-\bar\Delta_r \to 0.
\]
\end{proposition}

\begin{proof}
Define the plug-in remainder
\[
R_{N,r}
:=
\hat{\bar\Delta}_r^{\mathrm{FedIPW}}
-
\hat{\bar\Delta}_r^{\mathrm{IPW}}
=
\frac{1}{N}\sum_{i=1}^N
A_{i,r}\Delta_{i,r}
\left(
\frac{1}{\hat p_{i,r}}-\frac{1}{p_{i,r}}
\right).
\]
Also define
\[
e_N^{\mathrm{enroll}}
:=
\max_{1\le i\le N}
|\hat\pi_i^{\mathrm{enroll}}-\pi_i^{\mathrm{enroll}}|,
\qquad
e_{N,r}^{\mathrm{part}}
:=
\max_{1\le i\le N}
|\hat\pi_{i,r}^{\mathrm{part}}-\pi_{i,r}^{\mathrm{part}}|,
\]
and
\[
e_{N,r}
:=
\max_{1\le i\le N}
|\hat p_{i,r}-p_{i,r}|.
\]

\paragraph{Step 1: Consistency of the product propensity.}
For each $i$,
\[
\hat p_{i,r}-p_{i,r}
=
(\hat\pi_i^{\mathrm{enroll}}-\pi_i^{\mathrm{enroll}})
\hat\pi_{i,r}^{\mathrm{part}}
+
\pi_i^{\mathrm{enroll}}
(\hat\pi_{i,r}^{\mathrm{part}}-\pi_{i,r}^{\mathrm{part}}).
\]
Since all propensities lie in $[0,1]$,
\[
e_{N,r}
\le
e_N^{\mathrm{enroll}}+e_{N,r}^{\mathrm{part}}.
\]
By Assumption~\ref{assump:plugin-propensity-consistency},
\[
E[e_{N,r}]
\le
E[e_N^{\mathrm{enroll}}]
+
E[e_{N,r}^{\mathrm{part}}]
\to 0.
\]

\paragraph{Step 2: Remainder bound.}
By Assumption~\ref{assump:plugin-overlap},
\[
\left|
\frac{1}{\hat p_{i,r}}-\frac{1}{p_{i,r}}
\right|
=
\frac{|\hat p_{i,r}-p_{i,r}|}
{\hat p_{i,r}p_{i,r}}
\le
\frac{e_{N,r}}{c_r^2}.
\]
Therefore,
\[
\|R_{N,r}\|_2
\le
\frac{U_{N,r}e_{N,r}}{c_r^2}.
\]

\paragraph{Step 3: Vanishing plug-in bias.}
Since propensities lie in $[c_r,1]$, the error satisfies
$0\le e_{N,r}\le 1-c_r<1$, and hence
$e_{N,r}^2\le e_{N,r}$.
Thus
\[
E[e_{N,r}^2]\le E[e_{N,r}]\to0.
\]
By Cauchy--Schwarz and Assumption~\ref{assump:plugin-bounded-update},
\[
E\|R_{N,r}\|_2
\le
\frac{1}{c_r^2}
\{E[U_{N,r}^2]\}^{1/2}
\{E[e_{N,r}^2]\}^{1/2}
\to0.
\]
Therefore,
\[
\left\|
E[\hat{\bar\Delta}_r^{\mathrm{FedIPW}}]
-
E[\hat{\bar\Delta}_r^{\mathrm{IPW}}]
\right\|_2
\le
E\|R_{N,r}\|_2
\to0.
\]
Since $E[\hat{\bar\Delta}_r^{\mathrm{IPW}}]=\bar\Delta_r$, we obtain
\[
E[\hat{\bar\Delta}_r^{\mathrm{FedIPW}}]-\bar\Delta_r\to0.
\]
\end{proof}

\paragraph{Sufficient conditions for the bounded-update assumption.}
Assumption~\ref{assump:plugin-bounded-update} holds, for example, if local stochastic gradients are bounded:
\[
\Delta_{i,r}
=
-\eta_r\sum_{s=0}^{K_{i,r}-1}\mathbf g_{i,r}^{(s)},
\qquad
\|\mathbf g_{i,r}^{(s)}\|_2\le G_r,
\qquad
K_{i,r}\le K_{\max,r}.
\]
Then
\[
\|\Delta_{i,r}\|_2
\le
\eta_rK_{\max,r}G_r,
\]
so $U_{N,r}=O(1)$ deterministically.
More generally, Assumption~\ref{assump:plugin-bounded-update} holds if
\[
\frac{1}{N}\sum_{i=1}^N
E[A_{i,r}\|\Delta_{i,r}\|_2^2]
\le C_\Delta<\infty,
\]
because Jensen's inequality gives
\[
U_{N,r}^2
=
\left(
\frac{1}{N}\sum_{i=1}^N
A_{i,r}\|\Delta_{i,r}\|_2
\right)^2
\le
\frac{1}{N}\sum_{i=1}^N
A_{i,r}\|\Delta_{i,r}\|_2^2.
\]

\section{Formal Treatment of Aggregate-Calibrated FedIPW}
\label{app:aggregate-calibration}

This appendix formalizes the aggregate-calibration extension from Section~\ref{sec:limited-info-enrollment}.
This regime applies when pre-enrollment covariates are observed only for enrolled clients, while aggregate target-population summaries are available.

For each enrolled client $i\in\mathcal E:=\{i:E_i=1\}$, suppose we observe $Z_i$.
Let $b(Z_i)\in\mathbb R^q$ be a chosen balance map, and suppose an external source provides the target-population moment
\[
\mu_b=E[b(Z_i)].
\]
We construct normalized enrollment-calibration weights $q_i\ge0$, $i\in\mathcal E$, satisfying
\begin{equation}
\sum_{i\in\mathcal E}q_i=1,
\qquad
\sum_{i\in\mathcal E}q_i b(Z_i)=\mu_b.
\label{eq:appendix-calibration-constraints}
\end{equation}
For example, one may choose the feasible weights closest to uniform:
\begin{equation}
\min_{\{q_i\}_{i\in\mathcal E}}
\sum_{i\in\mathcal E}
\left(q_i-\frac{1}{N^{\mathrm{enroll}}}\right)^2
\quad
\text{subject to}
\quad
\sum_{i\in\mathcal E}q_i=1,
\quad
\sum_{i\in\mathcal E}q_i b(Z_i)=\mu_b,
\label{eq:appendix-quadratic-calibration}
\end{equation}
where $N^{\mathrm{enroll}}=|\mathcal E|$.
Other calibration objectives, such as entropy tilting, can also be used \citep{deville1992calibration}.

The aggregate-calibrated round update is
\begin{equation}
\hat{\bar\Delta}_r^{\mathrm{cal}}
:=
\sum_{i:A_{i,r}=1}
q_i\frac{\Delta_{i,r}}{\hat\pi_{i,r}^{\mathrm{part}}}.
\label{eq:appendix-calibrated-estimator}
\end{equation}

\begin{assumption}[Calibration feasibility]
\label{assump:calibration-feasibility}
The target-population moment vector $\mu_b$ lies in the convex hull of
$\{b(Z_i):i\in\mathcal E\}$.
\end{assumption}

\begin{assumption}[Participation ignorability among enrolled clients]
\label{assump:cal-participation-ignorability}
For every enrolled client $i$ and round $r$,
\[
E[\Delta_{i,r}\mid Z_i,X_{i,r},E_i=1,A_{i,r}=1]
=
E[\Delta_{i,r}\mid Z_i,X_{i,r},E_i=1].
\]
\end{assumption}

\begin{assumption}[Enrollment ignorability given balance covariates]
\label{assump:cal-enrollment-ignorability}
For every client $i$ and round $r$,
\[
E[\Delta_{i,r}\mid Z_i,E_i=1]
=
E[\Delta_{i,r}\mid Z_i].
\]
\end{assumption}

\begin{assumption}[Moment sufficiency]
\label{assump:cal-moment-sufficiency}
There exists a matrix $B_r$ such that
\[
E[\Delta_{i,r}\mid Z_i]
=
B_r b(Z_i).
\]
\end{assumption}

\begin{assumption}[Participation overlap and consistency]
\label{assump:cal-participation-overlap-consistency}
There exists $c_r>0$ such that, for all enrolled clients,
\[
\pi_{i,r}^{\mathrm{part}}\ge c_r,
\qquad
\hat\pi_{i,r}^{\mathrm{part}}\ge c_r,
\]
and
\[
E\!\left[
\max_{i:E_i=1}
\left|
\hat\pi_{i,r}^{\mathrm{part}}
-
\pi_{i,r}^{\mathrm{part}}
\right|
\right]\to0.
\]
\end{assumption}

\begin{assumption}[Bounded calibrated update moment]
\label{assump:cal-bounded-update}
\[
E\!\left[
\left(
\sum_{i:A_{i,r}=1}
q_i\|\Delta_{i,r}\|_2
\right)^2
\right]
=O(1).
\]
\end{assumption}

\begin{proposition}[Aggregate-calibration consistency under moment sufficiency]
\label{prop:bridge-consistency}
Fix a communication round $r$.
Let $\hat{\bar\Delta}_r^{\mathrm{cal}}$ be constructed as in~\eqref{eq:appendix-calibrated-estimator}.
Under Assumptions~\ref{assump:calibration-feasibility}--\ref{assump:cal-bounded-update},
\[
E[\hat{\bar\Delta}_r^{\mathrm{cal}}]-\bar\Delta_r\to0.
\]
\end{proposition}
\begin{proof}
Define the oracle aggregate-calibrated estimator
\[
\hat{\bar\Delta}_r^{\mathrm{cal},\star}
:=
\sum_{i:A_{i,r}=1}
q_i\frac{\Delta_{i,r}}{\pi_{i,r}^{\mathrm{part}}}.
\]
We first show that replacing the true participation propensities with estimated propensities has asymptotically negligible effect.
Let
\[
R_r
:=
\hat{\bar\Delta}_r^{\mathrm{cal}}
-
\hat{\bar\Delta}_r^{\mathrm{cal},\star}
=
\sum_{i:A_{i,r}=1}
q_i\Delta_{i,r}
\left(
\frac{1}{\hat\pi_{i,r}^{\mathrm{part}}}
-
\frac{1}{\pi_{i,r}^{\mathrm{part}}}
\right).
\]
Set
\[
e_r
:=
\max_{i:E_i=1}
\left|
\hat\pi_{i,r}^{\mathrm{part}}
-
\pi_{i,r}^{\mathrm{part}}
\right|,
\qquad
U_r^q
:=
\sum_{i:A_{i,r}=1}
q_i\|\Delta_{i,r}\|_2 .
\]
By participation overlap,
\[
\left|
\frac{1}{\hat\pi_{i,r}^{\mathrm{part}}}
-
\frac{1}{\pi_{i,r}^{\mathrm{part}}}
\right|
\le
\frac{e_r}{c_r^2}.
\]
Therefore,
\[
\|R_r\|_2
\le
\frac{U_r^q e_r}{c_r^2}.
\]
Since the propensities are bounded in $[c_r,1]$, we have $e_r^2\le e_r$.
By participation propensity consistency, $E[e_r]\to0$, and hence $E[e_r^2]\to0$.
Using Cauchy--Schwarz and Assumption~\ref{assump:cal-bounded-update},
\[
E\|R_r\|_2
\le
\frac{1}{c_r^2}
\{E[(U_r^q)^2]\}^{1/2}
\{E[e_r^2]\}^{1/2}
\to0.
\]
Thus,
\[
\left\|
E[\hat{\bar\Delta}_r^{\mathrm{cal}}]
-
E[\hat{\bar\Delta}_r^{\mathrm{cal},\star}]
\right\|_2
\le
E\|R_r\|_2
\to0.
\]

It remains to evaluate the oracle calibrated estimator.
Conditioning on the enrolled covariates and calibration weights,
\[
E[\hat{\bar\Delta}_r^{\mathrm{cal},\star}]
=
E\!\left[
\sum_{i:E_i=1}
q_i
E\!\left[
\frac{A_{i,r}\Delta_{i,r}}{\pi_{i,r}^{\mathrm{part}}}
\,\middle|\,
Z_i,E_i=1
\right]
\right].
\]
Conditioning further on $X_{i,r}$ and using participation ignorability among enrolled clients gives
\[
E\!\left[
\frac{A_{i,r}\Delta_{i,r}}{\pi_{i,r}^{\mathrm{part}}}
\,\middle|\,
Z_i,E_i=1
\right]
=
E[\Delta_{i,r}\mid Z_i,E_i=1].
\]
By enrollment ignorability given $Z_i$,
\[
E[\Delta_{i,r}\mid Z_i,E_i=1]
=
E[\Delta_{i,r}\mid Z_i].
\]
By moment sufficiency,
\[
E[\Delta_{i,r}\mid Z_i]
=
B_r b(Z_i).
\]
Therefore, using the calibration constraint,
\[
E[\hat{\bar\Delta}_r^{\mathrm{cal},\star}]
=
E\!\left[
B_r
\sum_{i:E_i=1} q_i b(Z_i)
\right]
=
B_r\mu_b.
\]
On the other hand, the target-population mean update satisfies
\[
\bar\Delta_r
=
E[\Delta_{i,r}]
=
E\!\left[E[\Delta_{i,r}\mid Z_i]\right]
=
E[B_r b(Z_i)]
=
B_r\mu_b.
\]
Hence
\[
E[\hat{\bar\Delta}_r^{\mathrm{cal},\star}]
=
\bar\Delta_r.
\]
Combining this equality with the estimated--oracle difference gives
\[
E[\hat{\bar\Delta}_r^{\mathrm{cal}}]-\bar\Delta_r\to0.
\]
\end{proof}

\section{Formal Optimization Analysis under Weight Misspecification}
\label{app:proof-conv}

This appendix gives the formal version of Theorem~\ref{thm:fedipw-bias-floor} and proves the finite-step bias-floor bound.
The proof has four ingredients:
a small toolbox of inequalities used throughout (Section~\ref{app:toolbox});
a decomposition of the FedIPW update into mean and noise components (Section~\ref{app:decomposition});
a bound on the local drift (Section~\ref{app:drift});
and a one-round function-value descent inequality that is unrolled to the final bound (Section~\ref{app:descent}).
We use $C>0$ for a universal constant whose value may change line to line.

\subsection{Setup and assumptions}
\label{app:conv-setup-assumptions}

At round $r$, client $i$ initializes $y_{i,0}=\theta_r$ and runs $K$ local SGD steps
\[
y_{i,k}
=
y_{i,k-1}-\eta g_i(y_{i,k-1}),
\qquad k=1,\dots,K,
\]
where $g_i(\cdot)$ is a stochastic gradient of $f_i$.
Only clients with $A_{i,r}=1$ return their final model difference.
The server update is
\[
\theta_{r+1}
=
\theta_r
+
\frac{\gamma}{N}
\sum_{i:A_{i,r}=1}
\frac{y_{i,K}-\theta_r}{\widehat p_{i,r}},
\]
where $\gamma\ge1$ is the server step size.
Equivalently, defining the effective step size
\[
\widetilde\eta:=K\gamma\eta,
\]
the update can be written as
\[
\theta_{r+1}
=
\theta_r
-
\frac{\widetilde\eta}{KN}
\sum_{i=1}^N
\sum_{k=1}^K
\frac{A_{i,r}}{\widehat p_{i,r}}
g_i(y_{i,k-1}).
\]

Let
\[
\rho_{i,r}
:=
\frac{p_{i,r}}{\widehat p_{i,r}}
\]
be the relative weight ratio, where $p_{i,r}$ is the true two-stage inclusion probability and $\widehat p_{i,r}$ is the probability used by the algorithm.

\begin{assumption}[Smoothness and strong convexity]
\label{assump:smooth-sc}
Each $f_i$ is $L$-smooth, and
\[
F(\theta)=\frac{1}{N}\sum_{i=1}^N f_i(\theta)
\]
is $\mu$-strongly convex.
We write $\kappa:=L/\mu$.
\end{assumption}

\begin{assumption}[Stochastic-gradient noise]
\label{assump:sgd-noise}
For every client $i$ and parameter $\theta$,
\[
E[g_i(\theta)\mid\theta]=\nabla f_i(\theta),
\qquad
E\|g_i(\theta)-\nabla f_i(\theta)\|^2\le\sigma^2.
\]
Stochastic-gradient noise is independent across clients conditional on the local parameters.
\end{assumption}

\begin{assumption}[Bounded gradient heterogeneity]
\label{assump:bgd-conv}
There exist constants $G\ge0$ and $B\ge1$ such that, for all $\theta$,
\[
\frac{1}{N}\sum_{i=1}^N
\|\nabla f_i(\theta)\|^2
\le
G^2+2LB^2(F(\theta)-F^\star).
\]
In particular,
\[
\frac{1}{N}\sum_{i=1}^N
\|\nabla f_i(\theta^\star)\|^2
\le
G^2.
\]
\end{assumption}

\begin{assumption}[Overlap and relative weight error]
\label{assump:overlap-conv}
There exist constants $p_{\min}>0$ and $\epsilon_w\in[0,1/2]$ such that, for all clients $i$ and rounds $r$,
\[
p_{i,r}\ge p_{\min},
\qquad
\widehat p_{i,r}\ge p_{\min},
\qquad
|\rho_{i,r}-1|\le\epsilon_w.
\]
\end{assumption}

\begin{assumption}[Selection independent of fresh randomness]
\label{assump:indep-conv}
Conditional on $(Z_i,X_{i,r})$, $\theta_r$, and $\{\widehat p_{i,r}\}_{i=1}^N$, the participation indicators $\{A_{i,r}\}_{i=1}^N$ are mutually independent and are jointly independent of the within-round stochastic-gradient noise.
\end{assumption}

\subsection{Formal theorem}
\label{app:conv-formal-theorem}

Define
\[
V
:=
\frac{\sigma^2}{KNp_{\min}}
+
\frac{L\sigma^2}{K\gamma^2}
+
\frac{G^2}{Np_{\min}}.
\]

\begin{theorem}[Bias-floor decomposition, formal]
\label{thm:fedipw-bias-floor-formal}
Under Assumptions~\ref{assump:smooth-sc}--\ref{assump:indep-conv}, set
\[
c_0:=8\left(1+\frac{\kappa B^2}{Np_{\min}}\right).
\]
Suppose
\[
\epsilon_w^2
\le
c_{\mathrm{abs}}\frac{\mu}{LB^2}
\]
for a sufficiently small universal constant $c_{\mathrm{abs}}>0$, and suppose the effective step size satisfies
\[
\widetilde\eta\le \frac{1}{c_0L}.
\]
Then the last iterate after $R$ communication rounds satisfies
\begin{equation}
E[F(\theta_R)-F^\star]
\le
\left(1-\frac{\mu\widetilde\eta}{8}\right)^R h_0
+
\frac{C\widetilde\eta V}{\mu}
+
\frac{C\epsilon_w^2G^2}{\mu},
\label{eq:app-formal-bias-floor-finite}
\end{equation}
where $h_0:=F(\theta_0)-F^\star$ and $C>0$ is a universal constant.
\end{theorem}

With the standard horizon-dependent step-size choice satisfying
$\widetilde\eta\le1/(c_0L)$, Theorem~\ref{thm:fedipw-bias-floor-formal} implies the order bound
\begin{equation}
E[F(\theta_R)-F^\star]
\le
\widetilde O\!\left(
\frac{V}{\mu R}
+
h_0\exp\!\left(-\frac{c\mu R}{L}\right)
+
\frac{\epsilon_w^2G^2}{\mu}
\right)
\label{eq:app-formal-bias-floor-rate}
\end{equation}
for a universal constant $c>0$.

\subsection{Notation and basic identities}
\label{app:notation}

We prove Theorem~\ref{thm:fedipw-bias-floor-formal}.
Fix a round $r$ and write
\[
  x:=\theta_r,
  \qquad
  h_r:=E[F(\theta_r)-F^\star],
  \qquad
  \widetilde\eta:=K\gamma\eta,
  \qquad
  \rho_i:=p_{i,r}/\widehat p_{i,r}.
\]
For each client $i\in\{1,\dots,N\}$, define ghost local trajectories $y_{i,0}=x$ and $y_{i,k}=y_{i,k-1}-\eta\,g_i(y_{i,k-1})$ for $k=1,\dots,K$.
Only clients with $A_i=1$ are observed by the server, but defining the trajectories for all $i$ is harmless.
The server update can then be written
\begin{equation}
  \theta_{r+1}=x-\frac{\widetilde\eta}{KN}\sum_{i,k}\frac{A_i}{\widehat p_i}\,g_i(y_{i,k-1}).
  \label{eq:appD-server-update}
\end{equation}
Let $\mathcal F_r$ be the sigma-field of all randomness up to the start of round $r$; in particular, $\theta_r$, $\{p_i\}$, and $\{\widehat p_i\}$ are $\mathcal F_r$-measurable.
Write $E_r[\,\cdot\,]:=E[\,\cdot\mid\mathcal F_r]$.
By Assumption~\ref{assump:overlap-conv},
\begin{equation}
  |\rho_i-1|\le\epsilon_w,
  \qquad
  \tfrac{1}{2}\le\rho_i\le\tfrac{3}{2},
  \qquad
  \widehat p_i,p_i\ge p_{\min}.
  \label{eq:appD-rho-bounds}
\end{equation}

We define three derived quantities:
\begin{align}
  \overline g_r
  &:=
  E_r\!\left[\frac{1}{KN}\sum_{i,k}\nabla f_i(y_{i,k-1})\right]
  &&\text{(mean path gradient),}
  \label{eq:appD-bar-g}
  \\
  b_r
  &:=
  E_r\!\left[\frac{1}{KN}\sum_{i,k}(\rho_i-1)\,\nabla f_i(y_{i,k-1})\right]
  &&\text{(weight-misspecification bias),}
  \label{eq:appD-b}
  \\
  \xi_r
  &:=
  (\theta_{r+1}-x)-E_r[\theta_{r+1}-x]
  &&\text{(centered update noise).}
  \label{eq:appD-xi}
\end{align}
The local-drift quantity is
\[
  \mathcal E_r:=\frac{1}{KN}\sum_{i,k}E\|y_{i,k-1}-x\|^2.
\]

\subsection{Toolbox: three reusable inequalities}
\label{app:toolbox}

We collect three inequalities that will be invoked repeatedly.

\begin{lemma}[Toolbox]
\label{lem:toolbox}
For every round $r$, parameter $x=\theta_r$, and trajectory point $y$:
\begin{enumerate}[label=\textup{(T\arabic*)},leftmargin=2em]
\item \emph{Local-to-anchor gradient bound.} By $L$-smoothness,
\begin{equation}
  \|\nabla f_i(y)\|^2 \le 2\|\nabla f_i(x)\|^2 + 2L^2\|y-x\|^2.
  \label{eq:T1-local-anchor}
\end{equation}

\item \emph{Population gradient bound.} By Assumption~\ref{assump:bgd-conv},
\begin{equation}
  \frac{1}{N}\sum_{i=1}^N E\|\nabla f_i(\theta_r)\|^2 \le G^2 + 2LB^2 h_r.
  \label{eq:T2-pop-gradient}
\end{equation}

\item \emph{Martingale variance.}
For any martingale-difference sequence $\{\zeta_s\}_{s=0}^{K-1}$ with $E\|\zeta_s\|^2\le\sigma^2$,
\begin{equation}
  E\Bigl\|\sum_{s=0}^{k-1}\zeta_s\Bigr\|^2 \le k\sigma^2,\qquad k\le K.
  \label{eq:T3-martingale}
\end{equation}
\end{enumerate}
\end{lemma}

\begin{proof}
(T1) follows from $\|\nabla f_i(y)-\nabla f_i(x)\|^2\le L^2\|y-x\|^2$ and $\|a\|^2\le 2\|a-b\|^2+2\|b\|^2$.
(T2) is Assumption~\ref{assump:bgd-conv} after taking expectation.
(T3) follows because $E\langle\zeta_s,\zeta_t\rangle=0$ for $s<t$ by the tower property, so cross-terms in the expansion of $\|\sum_s\zeta_s\|^2$ vanish.
\end{proof}

The combination of (T1) and (T2) appears repeatedly:
\begin{equation}
  \frac{1}{KN}\sum_{i,k} E\|\nabla f_i(y_{i,k-1})\|^2
  \le 2G^2 + 4LB^2 h_r + 2L^2\,\mathcal E_r.
  \label{eq:T1T2-combined}
\end{equation}

\subsection{Decomposition: mean update and centered noise}
\label{app:decomposition}

The first step is to decompose the FedIPW update into its conditional mean and a centered noise term, identifying the weight-misspecification bias inside the mean.

\begin{lemma}[Update decomposition]
\label{lem:appD-decomposition}
Under Assumptions~\ref{assump:sgd-noise}, \ref{assump:overlap-conv}, and \ref{assump:indep-conv},
\begin{align}
  E_r[\theta_{r+1}-x]
  &=
  -\widetilde\eta(\overline g_r+b_r),
  \label{eq:appD-conditional-mean}
  \\
  E\|\xi_r\|^2
  &\le
  \frac{C\widetilde\eta^2\sigma^2}{KNp_{\min}}
  +\frac{C\widetilde\eta^2}{Np_{\min}}\bigl(G^2+LB^2 h_r+L^2\mathcal E_r\bigr).
  \label{eq:appD-centered-variance}
\end{align}
\end{lemma}

\begin{proof}
\textbf{Conditional mean.}
Starting from~\eqref{eq:appD-server-update},
\[
  E_r[\theta_{r+1}-x]
  =
  -\frac{\widetilde\eta}{KN}\sum_{i,k}E_r\!\left[\frac{A_i}{\widehat p_i}\,g_i(y_{i,k-1})\right].
\]
Conditional on $\mathcal F_r$, the weight $\widehat p_i$ is fixed.
Assumption~\ref{assump:indep-conv} states that, given $\mathcal F_r$, the indicator $A_i$ is independent of the within-round SGD noise.
Therefore
\[
E_r[A_i\,g_i(y_{i,k-1})]
=
p_i\,E_r[\nabla f_i(y_{i,k-1})],
\]
where we also use unbiasedness of the stochastic gradient.
Dividing by $\widehat p_i$,
\begin{equation}
  E_r\!\left[\frac{A_i}{\widehat p_i}\,g_i(y_{i,k-1})\right]
  =
  \rho_i\,E_r[\nabla f_i(y_{i,k-1})].
  \label{eq:appD-key-factorization}
\end{equation}
Writing $\rho_i=1+(\rho_i-1)$ and applying the definitions of $\overline g_r$ and $b_r$ proves~\eqref{eq:appD-conditional-mean}.

\textbf{Centered noise variance.}
Set
\[
  V_i:=\frac{1}{K}\sum_k g_i(y_{i,k-1}),
  \qquad
  \overline V_i:=\frac{1}{K}\sum_k\nabla f_i(y_{i,k-1}).
\]
Then $\theta_{r+1}-x=-(\widetilde\eta/N)\sum_i(A_i/\widehat p_i)V_i$ and the conditional mean is $-(\widetilde\eta/N)\sum_i\rho_i\overline V_i$.
Subtracting,
\[
  \xi_r
  =
  -\frac{\widetilde\eta}{N}\sum_i\biggl[
    \frac{A_i}{\widehat p_i}(V_i-\overline V_i)
    +
    \Bigl(\frac{A_i}{\widehat p_i}-\rho_i\Bigr)\overline V_i
  \biggr]
  =: -\widetilde\eta\,(T_1+T_2).
\]
The two terms are uncorrelated:
the first averages to zero given the trajectories, while the second averages to zero given the SGD noise.
Hence
\[
E\|\xi_r\|^2
=
\widetilde\eta^2 E\|T_1\|^2
+
\widetilde\eta^2 E\|T_2\|^2.
\]

\emph{Bound on $\|T_1\|^2$.}
Conditional on the trajectories, $V_i-\overline V_i$ is the average of the SGD-noise martingale-difference sequence $\zeta_{i,k}:=g_i(y_{i,k-1})-\nabla f_i(y_{i,k-1})$.
By (T3) of Lemma~\ref{lem:toolbox},
$E_r\|V_i-\overline V_i\|^2\le\sigma^2/K$.
By Assumptions~\ref{assump:indep-conv} and \ref{assump:sgd-noise}, $T_1$ is a sum of mutually uncorrelated terms with
\[
E_r[A_i^2/\widehat p_i^2]
=
p_i/\widehat p_i^2
=
\rho_i/\widehat p_i
\le
(3/2)/p_{\min}.
\]
Therefore
\begin{equation}
  E\|T_1\|^2
  =
  \frac{1}{N^2}\sum_i E\!\left[\frac{p_i}{\widehat p_i^2}\,\|V_i-\overline V_i\|^2\right]
  \le
  \frac{C\sigma^2}{KNp_{\min}}.
  \label{eq:appD-T1-bound}
\end{equation}

\emph{Bound on $\|T_2\|^2$.}
$T_2$ is a sum over clients of independent terms with mean zero.
For each $i$,
\[
E_r[(A_i/\widehat p_i-\rho_i)^2]
=
p_i(1-p_i)/\widehat p_i^2
\le
1/\widehat p_i
\le
1/p_{\min}.
\]
Hence
\[
  E\|T_2\|^2
  \le
  \frac{1}{N^2 p_{\min}}\sum_i E\|\overline V_i\|^2.
\]
Jensen's inequality gives
\[
\|\overline V_i\|^2
\le
\frac{1}{K}\sum_k\|\nabla f_i(y_{i,k-1})\|^2,
\]
and the combined toolbox bound~\eqref{eq:T1T2-combined} gives
\[
\frac{1}{N}\sum_i E\|\overline V_i\|^2
\le
2G^2+4LB^2 h_r+2L^2\mathcal E_r.
\]
Therefore
\begin{equation}
  E\|T_2\|^2
  \le
  \frac{C}{Np_{\min}}\bigl(G^2+LB^2 h_r+L^2\mathcal E_r\bigr).
  \label{eq:appD-T2-bound}
\end{equation}
Multiplying~\eqref{eq:appD-T1-bound}--\eqref{eq:appD-T2-bound} by $\widetilde\eta^2$ and combining gives~\eqref{eq:appD-centered-variance}.
\end{proof}

\begin{lemma}[Bias bound]
\label{lem:appD-bias-bound}
Under Assumptions~\ref{assump:smooth-sc}, \ref{assump:bgd-conv}, and \ref{assump:overlap-conv},
\begin{equation}
  E\|b_r\|^2 \le 2\,\epsilon_w^2\bigl(G^2+2LB^2 h_r+L^2\mathcal E_r\bigr).
  \label{eq:appD-bias-bound}
\end{equation}
\end{lemma}

\begin{proof}
By Cauchy--Schwarz applied to the definition of $b_r$,
\[
  \|b_r\|^2
  \le
  \frac{1}{N}\sum_i(\rho_i-1)^2\cdot\frac{1}{K}\sum_k\|\nabla f_i(y_{i,k-1})\|^2
  \le
  \epsilon_w^2\cdot\frac{1}{KN}\sum_{i,k}\|\nabla f_i(y_{i,k-1})\|^2.
\]
Taking expectations and applying~\eqref{eq:T1T2-combined} gives the result.
\end{proof}

\subsection{Local drift}
\label{app:drift}

\begin{lemma}[Local drift]
\label{lem:appD-drift}
Suppose $\eta LK\le 1/4$.
Then
\begin{equation}
  \mathcal E_r \le \frac{C\widetilde\eta^2}{\gamma^2}\Bigl(\frac{\sigma^2}{K}+G^2+LB^2 h_r\Bigr).
  \label{eq:appD-drift-bound}
\end{equation}
\end{lemma}

\begin{proof}
Fix $i$ and let $d_k:=E\|y_{i,k}-x\|^2$.
Decompose $g_i(y_{i,s})=\nabla f_i(y_{i,s})+\zeta_{i,s}$, where $\zeta_{i,s}$ is the within-client SGD-noise martingale difference.
Then
\[
  y_{i,k}-x
  =
  -\eta\sum_{s<k}\nabla f_i(y_{i,s})
  -
  \eta\sum_{s<k}\zeta_{i,s}.
\]
Using $\|a+b\|^2\le 2\|a\|^2+2\|b\|^2$,
\[
  d_k
  \le
  2\eta^2 E\Bigl\|\sum_{s<k}\nabla f_i(y_{i,s})\Bigr\|^2
  +
  2\eta^2 E\Bigl\|\sum_{s<k}\zeta_{i,s}\Bigr\|^2.
\]
By (T3), the noise term is at most $2k\eta^2\sigma^2$.
By Jensen and (T1),
\[
  E\Bigl\|\sum_{s<k}\nabla f_i(y_{i,s})\Bigr\|^2
  \le
  k\sum_{s<k}E\|\nabla f_i(y_{i,s})\|^2
  \le
  k\sum_{s<k}\bigl(2 E\|\nabla f_i(x)\|^2+2L^2 d_s\bigr).
\]
Combining,
\begin{equation}
  d_k \le 4\eta^2 k^2 E\|\nabla f_i(x)\|^2 + 4\eta^2 L^2 k\sum_{s<k}d_s + 2\eta^2 k\sigma^2.
  \label{eq:appD-drift-recursion}
\end{equation}

We claim $d_k\le 16\eta^2(k^2 E\|\nabla f_i(x)\|^2+k\sigma^2)$ by induction on $k$, using $\eta^2 L^2 K^2\le 1/16$.
The base case $d_0=0$ is immediate.
For the inductive step, $\sum_{s<k}s^2\le k^3$ and $\sum_{s<k}s\le k^2$ give
\[
  \sum_{s<k}d_s
  \le
  16\eta^2(k^3 E\|\nabla f_i(x)\|^2+k^2\sigma^2).
\]
Substituting into~\eqref{eq:appD-drift-recursion},
\[
  d_k
  \le
  4\eta^2 k^2 E\|\nabla f_i(x)\|^2
  +
  64\eta^4 L^2(k^4 E\|\nabla f_i(x)\|^2+k^3\sigma^2)
  +
  2\eta^2 k\sigma^2.
\]
Using $\eta^2 L^2 k^2\le\eta^2 L^2 K^2\le 1/16$, the second term is at most
$4\eta^2 k^2 E\|\nabla f_i(x)\|^2+4\eta^2 k\sigma^2$.
Thus
\[
d_k
\le
8\eta^2 k^2 E\|\nabla f_i(x)\|^2+6\eta^2 k\sigma^2
\le
16\eta^2(k^2 E\|\nabla f_i(x)\|^2+k\sigma^2),
\]
closing the induction.

Averaging over $k\le K$ and clients, then applying (T2),
\[
  \mathcal E_r
  \le
  16\eta^2\bigl(K^2(G^2+2LB^2 h_r)+K\sigma^2\bigr).
\]
Substituting $\eta^2 K^2=\widetilde\eta^2/\gamma^2$ and $\eta^2 K=\widetilde\eta^2/(K\gamma^2)$ yields~\eqref{eq:appD-drift-bound}.
\end{proof}

\subsection{One-round descent and unrolling}
\label{app:descent}

\begin{lemma}[One-round descent]
\label{lem:appD-descent}
Under the assumptions of Theorem~\ref{thm:fedipw-bias-floor-formal},
\begin{equation}
  h_{r+1}
  \le
  \Bigl(1-\frac{\mu\widetilde\eta}{8}\Bigr)h_r
  +
  C\widetilde\eta^2 V
  +
  C\widetilde\eta\,\epsilon_w^2 G^2.
  \label{eq:appD-one-round-descent}
\end{equation}
\end{lemma}

\begin{proof}
By $L$-smoothness applied to $\theta_{r+1}=x+(\theta_{r+1}-x)$,
\[
  F(\theta_{r+1})
  \le
  F(x)+\langle\nabla F(x),\theta_{r+1}-x\rangle+\tfrac{L}{2}\|\theta_{r+1}-x\|^2.
\]
Taking $E_r$ and using~\eqref{eq:appD-conditional-mean},
\begin{equation}
  E_r[F(\theta_{r+1})]-F(x)
  \le
  -\widetilde\eta\langle\nabla F(x),\overline g_r+b_r\rangle
  +
  \tfrac{L}{2}E_r\|\theta_{r+1}-x\|^2.
  \label{eq:appD-smoothness-applied}
\end{equation}
We control the cross term and the quadratic term separately.

\textbf{Step 1: cross-term + Polyak--{\L}ojasiewicz.}
Write $\overline g_r=\nabla F(x)+\delta_r$, where $\delta_r:=\overline g_r-\nabla F(x)$.
Jensen's inequality combined with $L$-smoothness gives
\begin{equation}
  E\|\delta_r\|^2 \le L^2\,\mathcal E_r.
  \label{eq:appD-delta-bound}
\end{equation}
Apply Young's inequality $|\langle a,b\rangle|\le\tfrac{1}{4}\|a\|^2+\|b\|^2$ twice with $a=\nabla F(x)$:
\begin{align*}
  -\widetilde\eta\langle\nabla F(x),\delta_r\rangle
  &\le
  \tfrac{\widetilde\eta}{4}\|\nabla F(x)\|^2+\widetilde\eta\|\delta_r\|^2,\\
  -\widetilde\eta\langle\nabla F(x),b_r\rangle
  &\le
  \tfrac{\widetilde\eta}{4}\|\nabla F(x)\|^2+\widetilde\eta\|b_r\|^2.
\end{align*}
Adding these to $-\widetilde\eta\|\nabla F(x)\|^2$ from $\overline g_r=\nabla F(x)+\delta_r$,
\begin{equation}
  -\widetilde\eta\langle\nabla F(x),\overline g_r+b_r\rangle
  \le
  -\frac{\widetilde\eta}{2}\|\nabla F(x)\|^2
  +
  \widetilde\eta\|\delta_r\|^2
  +
  \widetilde\eta\|b_r\|^2.
  \label{eq:appD-young-applied}
\end{equation}
Strong convexity of $F$ implies the Polyak--{\L}ojasiewicz inequality
\[
\|\nabla F(x)\|^2\ge 2\mu(F(x)-F^\star).
\]
Thus
\begin{equation}
  -\frac{\widetilde\eta}{2}\|\nabla F(x)\|^2
  \le
  -\mu\widetilde\eta\bigl(F(x)-F^\star\bigr).
  \label{eq:appD-pl-applied}
\end{equation}
Combining~\eqref{eq:appD-young-applied} and~\eqref{eq:appD-pl-applied}, taking outer expectation, and applying~\eqref{eq:appD-delta-bound} and Lemma~\ref{lem:appD-bias-bound},
\begin{equation}
  E\bigl[-\widetilde\eta\langle\nabla F(x),\overline g_r+b_r\rangle\bigr]
  \le
  -\mu\widetilde\eta\,h_r
  +
  C\widetilde\eta L^2\mathcal E_r
  +
  C\widetilde\eta\,\epsilon_w^2(G^2+LB^2 h_r+L^2\mathcal E_r).
  \label{eq:appD-cross-final}
\end{equation}

\textbf{Step 2: quadratic term.}
Since $\theta_{r+1}-x=-\widetilde\eta(\overline g_r+b_r)+\xi_r$ and $E_r[\xi_r]=0$, the cross term vanishes and
\[
  E_r\|\theta_{r+1}-x\|^2
  =
  \widetilde\eta^2\|\overline g_r+b_r\|^2
  +
  E_r\|\xi_r\|^2.
\]
Using $\|\overline g_r+b_r\|^2\le 3\|\nabla F(x)\|^2+3\|\delta_r\|^2+3\|b_r\|^2$, taking outer expectation, and applying (T2), Lemma~\ref{lem:appD-bias-bound}, $E\|\delta_r\|^2\le L^2\mathcal E_r$, and~\eqref{eq:appD-centered-variance},
\begin{align}
  \tfrac{L}{2}E\|\theta_{r+1}-x\|^2
  &\le
  CL\widetilde\eta^2(G^2+LB^2 h_r)
  +
  CL^3\widetilde\eta^2\mathcal E_r
  \nonumber\\
  &\quad
  +
  CL\widetilde\eta^2\epsilon_w^2(G^2+LB^2 h_r+L^2\mathcal E_r)
  \nonumber\\
  &\quad
  +
  C\widetilde\eta^2\!\left[
  \frac{\sigma^2}{KNp_{\min}}
  +
  \frac{G^2+LB^2 h_r+L^2\mathcal E_r}{Np_{\min}}
  \right].
  \label{eq:appD-quadratic-final}
\end{align}

\textbf{Step 3: substitute the drift bound and absorb $h_r$ terms.}
Combining~\eqref{eq:appD-cross-final} and~\eqref{eq:appD-quadratic-final} produces several terms proportional to $h_r$.
We now show each is dominated by $\mu\widetilde\eta h_r$ up to a small constant fraction, so that they can be absorbed into the descent term.

\emph{Bias-induced.}
The contribution from $E\|b_r\|^2$ in~\eqref{eq:appD-cross-final} contains $C\widetilde\eta\epsilon_w^2 LB^2 h_r$.
By the assumption $\epsilon_w^2\le c_{\mathrm{abs}}\mu/(LB^2)$, this is at most
$C c_{\mathrm{abs}}\mu\widetilde\eta h_r\le\frac{\mu\widetilde\eta}{16}h_r$
for $c_{\mathrm{abs}}$ small enough.

\emph{Sampling-induced.}
The contribution $\frac{CL^2 B^2\widetilde\eta^2}{Np_{\min}}h_r$ from~\eqref{eq:appD-quadratic-final}, after using $L\widetilde\eta\le 1/c_0$ once, is bounded by
\[
  \frac{CLB^2\widetilde\eta}{c_0 Np_{\min}}h_r
  \le
  \frac{\mu\widetilde\eta}{16}h_r,
\]
where the last inequality uses
\[
c_0\ge \frac{8\kappa B^2}{Np_{\min}}
=
\frac{8LB^2}{\mu Np_{\min}}.
\]

\emph{Drift-induced.}
The leading drift contribution $\widetilde\eta L^2\mathcal E_r$ in~\eqref{eq:appD-cross-final}, after substituting~\eqref{eq:appD-drift-bound} and extracting the $h_r$ component, gives
\[
  C\widetilde\eta L^2\cdot\frac{\widetilde\eta^2}{\gamma^2}\,LB^2 h_r
  \le
  C\widetilde\eta(L\widetilde\eta)^2 LB^2 h_r
  \le
  \frac{C\widetilde\eta LB^2}{c_0^2}h_r,
\]
which is dominated by the sampling-induced bound above.
The remaining drift terms in~\eqref{eq:appD-quadratic-final} are dominated by combinations of the previous two cases.

The absorbed contributions sum to at most a constant fraction of $\mu\widetilde\eta h_r$, leaving the contraction
\[
-\frac{\mu\widetilde\eta}{8}h_r.
\]

\textbf{Step 4: collect non-$h_r$ terms.}
The remaining terms have the following structure:
\begin{itemize}[leftmargin=2em]
\item \emph{Stochastic-gradient noise:} $C\widetilde\eta^2\sigma^2/(KNp_{\min})$ from~\eqref{eq:appD-centered-variance}.
\item \emph{Local-drift stochastic part:} $C\widetilde\eta L^2\cdot\widetilde\eta^2\sigma^2/(K\gamma^2)\le C\widetilde\eta^2L\sigma^2/(K\gamma^2)$ after using $L\widetilde\eta\le 1/c_0$.
\item \emph{Client-sampling:} $C\widetilde\eta^2 G^2/(Np_{\min})$ from~\eqref{eq:appD-T2-bound}.
\item \emph{Bias floor:} $C\widetilde\eta\epsilon_w^2 G^2$ from the $G^2$ part of $E\|b_r\|^2$ in~\eqref{eq:appD-cross-final}.
\end{itemize}
The first three sum to $C\widetilde\eta^2 V$.
Hence
\[
  h_{r+1}
  \le
  \Bigl(1-\frac{\mu\widetilde\eta}{8}\Bigr)h_r
  +
  C\widetilde\eta^2 V
  +
  C\widetilde\eta\,\epsilon_w^2 G^2,
\]
which proves~\eqref{eq:appD-one-round-descent}.
\end{proof}

\paragraph{Unrolling the recursion.}
Set
\[
\alpha:=\frac{\mu\widetilde\eta}{8},
\qquad
\varepsilon:=C\widetilde\eta^2 V+C\widetilde\eta\epsilon_w^2 G^2.
\]
Lemma~\ref{lem:appD-descent} states that
\[
h_{r+1}\le(1-\alpha)h_r+\varepsilon.
\]
Iterating,
\[
  h_R
  \le
  (1-\alpha)^R h_0+\varepsilon\sum_{t=0}^{R-1}(1-\alpha)^t
  \le
  (1-\alpha)^R h_0+\frac{\varepsilon}{\alpha}.
\]
Substituting $\alpha$ and $\varepsilon$ gives
\[
  h_R
  \le
  \Bigl(1-\frac{\mu\widetilde\eta}{8}\Bigr)^R h_0
  +
  \frac{C\widetilde\eta V}{\mu}
  +
  \frac{C\epsilon_w^2 G^2}{\mu},
\]
which proves~\eqref{eq:app-formal-bias-floor-finite}.
The order bound~\eqref{eq:app-formal-bias-floor-rate} follows from the standard horizon-dependent step-size choice satisfying $\widetilde\eta\le1/(c_0L)$ and the inequality $(1-\alpha)^R\le\exp(-\alpha R)$.

\subsection{Recovery of standard FedAvg}
\label{app:recovery-fedavg}

Setting $\epsilon_w=0$ and $p_i=S/N$ gives $b_r=0$ and $Np_{\min}=S$.
The bound reduces to
\[
  E[F(\theta_R)-F^\star]
  \le
  \widetilde O\!\left(
    \frac{\sigma^2}{\mu KSR}
    +
    \frac{\sigma^2}{\mu K\gamma^2 R}
    +
    \frac{G^2}{\mu SR}
    +
    h_0\exp\!\left(-\frac{c\mu R}{L}\right)
  \right),
\]
which matches the strongly-convex FedAvg/SCAFFOLD rate of \citet{karimireddy2020scaffold} up to lower-order constants reflecting the difference between Bernoulli and fixed-size sampling.

\subsection{Enrollment omission as structural misspecification}
\label{app:participation-only-corollary}

\begin{corollary}[Enrollment omission leaves residual weight error]
\label{cor:participation-only}
Consider a participation-only IPW method that uses
\[
\widehat p_{i,r}
=
\widehat\pi_{i,r}^{\mathrm{part}}.
\]
If the participation propensities are known exactly,
$\widehat\pi_{i,r}^{\mathrm{part}}=\pi_{i,r}^{\mathrm{part}}$,
then
\[
\rho_{i,r}
=
\frac{\pi_i^{\mathrm{enroll}}\pi_{i,r}^{\mathrm{part}}}
{\pi_{i,r}^{\mathrm{part}}}
=
\pi_i^{\mathrm{enroll}}.
\]
Thus the residual weight error is
\[
\epsilon_w^{\mathrm{part\text{-}only}}
=
\max_i |1-\pi_i^{\mathrm{enroll}}|.
\]
If this quantity satisfies the small-error conditions of Theorem~\ref{thm:fedipw-bias-floor-formal}, the formal bound contains the residual term
\[
\frac{C(\epsilon_w^{\mathrm{part\text{-}only}})^2G^2}{\mu}.
\]
This term is independent of participation-propensity estimation error and reflects structural omission of the enrollment mechanism.
\end{corollary}

\section{Order-sharpness within the weighted-objective class}
\label{app:weighted-objective-lower-bound}

Theorem~\ref{thm:fedipw-bias-floor-formal} gives an upper bound with a persistent term of order
\[
\frac{\epsilon_w^2G^2}{\mu}.
\]
We show that this dependence is unavoidable within the natural class of methods whose limiting expected update is centered on a misspecified selection-weighted objective.

\begin{proposition}[Order-sharpness within the weighted-objective class]
\label{prop:weighted-objective-lower-bound}
Fix $\epsilon_w\in(0,1/2]$.
There exists a two-client FL instance satisfying the smoothness, strong-convexity, and bounded-gradient-heterogeneity assumptions of Theorem~\ref{thm:fedipw-bias-floor-formal}, together with a misspecified weighted objective
\[
F_\rho(\theta)
=
\frac{\rho_1 f_1(\theta)+\rho_2 f_2(\theta)}
{\rho_1+\rho_2},
\qquad
\max_i|\rho_i-1|=\epsilon_w,
\]
such that any FL method whose expected limiting iterate minimizes $F_\rho$ has target-population suboptimality
\[
F(\theta_\rho^\star)-F^\star
\ge
\frac{1}{8}
\frac{\epsilon_w^2G^2}{\mu},
\]
where
\[
\theta_\rho^\star:=\arg\min_\theta F_\rho(\theta).
\]
In particular, participation-only IPW realizes this misspecified weighted objective when enrollment is heterogeneous and the enrollment factor is omitted.
\end{proposition}

\begin{proof}
Fix $\epsilon_w\in(0,1/2]$.
Consider a one-dimensional two-client instance with
\[
f_1(\theta)=\frac{\mu}{2}(\theta-a)^2,
\qquad
f_2(\theta)=\frac{\mu}{2}(\theta+a)^2,
\]
where $a>0$ will determine the gradient-heterogeneity level.
The target-population objective is
\[
F(\theta)
=
\frac{1}{2}\{f_1(\theta)+f_2(\theta)\}
=
\frac{\mu}{4}\{(\theta-a)^2+(\theta+a)^2\}
=
\frac{\mu}{2}(\theta^2+a^2).
\]
Thus $F$ is $\mu$-strongly convex, each $f_i$ is $\mu$-smooth, and the target minimizer is
\[
\theta^\star=0.
\]

We first verify the bounded-gradient-heterogeneity condition.
For this instance,
\[
\nabla f_1(\theta)=\mu(\theta-a),
\qquad
\nabla f_2(\theta)=\mu(\theta+a).
\]
Therefore
\[
\frac{1}{2}
\left(
\|\nabla f_1(\theta)\|^2
+
\|\nabla f_2(\theta)\|^2
\right)
=
\frac{\mu^2}{2}
\left\{(\theta-a)^2+(\theta+a)^2\right\}
=
\mu^2(\theta^2+a^2).
\]
Also,
\[
F(\theta)-F^\star
=
\frac{\mu}{2}\theta^2.
\]
Hence the bounded-gradient-heterogeneity assumption
\[
\frac{1}{2}
\sum_{i=1}^2
\|\nabla f_i(\theta)\|^2
\le
G^2+2LB^2(F(\theta)-F^\star)
\]
holds with $L=\mu$, $B=1$, and the minimal admissible choice
\[
G^2=\mu^2a^2.
\]
Indeed,
\[
G^2+2LB^2(F(\theta)-F^\star)
=
\mu^2a^2
+
2\mu\cdot \frac{\mu}{2}\theta^2
=
\mu^2(a^2+\theta^2),
\]
which matches the left-hand side.

Now choose misspecified relative weights
\[
\rho_1=1+\epsilon_w,
\qquad
\rho_2=1-\epsilon_w.
\]
Then
\[
\max_i|\rho_i-1|=\epsilon_w.
\]
The corresponding weighted objective is
\[
F_\rho(\theta)
=
\frac{\rho_1 f_1(\theta)+\rho_2 f_2(\theta)}
{\rho_1+\rho_2}.
\]
Its derivative is
\[
\nabla F_\rho(\theta)
=
\frac{\mu}{\rho_1+\rho_2}
\left\{
\rho_1(\theta-a)+\rho_2(\theta+a)
\right\}.
\]
Setting this derivative to zero gives
\[
\rho_1(\theta-a)+\rho_2(\theta+a)=0,
\]
so
\[
(\rho_1+\rho_2)\theta
=
(\rho_1-\rho_2)a.
\]
Since $\rho_1+\rho_2=2$ and $\rho_1-\rho_2=2\epsilon_w$, the minimizer is
\[
\theta_\rho^\star
=
\epsilon_w a.
\]

The target-population suboptimality at this misspecified minimizer is
\[
F(\theta_\rho^\star)-F^\star
=
\frac{\mu}{2}(\theta_\rho^\star)^2
=
\frac{\mu}{2}\epsilon_w^2a^2.
\]
Using the minimal admissible heterogeneity constant $G^2=\mu^2a^2$, we obtain
\[
F(\theta_\rho^\star)-F^\star
=
\frac{1}{2}
\frac{\epsilon_w^2G^2}{\mu}
\]
This proves the claimed lower bound.

Finally, this weighted objective is realized by participation-only IPW when enrollment is heterogeneous and omitted.
Indeed, suppose the true inclusion probability factorizes as
\[
p_i=\pi_i^{\mathrm{enroll}}\pi_i^{\mathrm{part}},
\]
but the algorithm corrects only participation and uses
\[
\widehat p_i=\pi_i^{\mathrm{part}}.
\]
Then the relative weighting factor is
\[
\rho_i=\frac{p_i}{\widehat p_i}
=
\pi_i^{\mathrm{enroll}}.
\]
Choosing heterogeneous enrollment factors proportional to
$1+\epsilon_w$ and $1-\epsilon_w$ therefore yields the same misspecified weighted objective, up to a common normalization factor.
\end{proof}

\end{document}